\documentclass{article}
\usepackage[margin=1.25in]{geometry}
\usepackage{graphicx}
\usepackage{amsmath,mathtools,amsthm}
\usepackage{amssymb}
\usepackage{amsbsy}
\usepackage{amsfonts}
\usepackage{xcolor}
\usepackage{algorithm, algorithmicx, algpseudocode}
\usepackage[justification=justified,singlelinecheck=false,justification=justified]{caption}
\usepackage[hang,flushmargin]{footmisc} 
\makeatletter
\newenvironment{mytable}
{\def\@captype{table}}
{}
\makeatother
\let\emptyset\varnothing
\let\epsilon\varepsilon
\let\phi\varphi

\def\O{{\mathcal O}}

\def\x{{\boldsymbol{x}}}

\def\b{{\boldsymbol{b}}}

\def\y{{\boldsymbol{y}}}

\def\R{\mathbb R}
\def\N{\mathbb N}

\def\B{\mathcal B}
\def\X{\mathcal X}
\def\T{\mathcal T}
\def\U{\mathcal N}
\def\as{\text{a.s.}}
\def\st{\text{s.t.}}

\def\-as{\text{-a.s.}}
\def\argmax{\operatorname{argmax}}

\newcommand{\lft}[1]{L({#1})}
\newcommand{\rt}[1]{R({#1})}

\newcommand{\llft}[1]{L({#1})}
\newcommand{\rrt}[1]{R({#1})}
\newcommand{\mingap}{\lambda_{\min}}
\newcommand{\gap}{\lambda}

\newcommand{\sort}{\textbf {sort}}
\newcommand{\dig}[3]{\left| \nu(X_{#2..#3},B)-{#1}(B) \right|}
\newcommand{\argmaxdisp}[1]{\underset{#1}{\argmax~}}
\newcommand{\argmaxdispt}[2]{{\underset{#1}{\overset{#2}{\argmax~}}}}

\newcommand{\maxdisp}[1]{\underset{#1}{\max~}}
\newcommand{\mindisp}[1]{\underset{#1}{\min~}}

\newtheorem{thm}{Theorem}
\newtheorem{lem}{Lemma}
\newtheorem{defn}{Definition}

\newtheorem{prop}{Proposition}
\definecolor{darkblue}{RGB}{0,0,128}
\definecolor{darkred}{RGB}{128,0,0}
\definecolor{darkgreen}{RGB}{0,128,0}
\definecolor{mathred}{RGB}{204,0,0}
\definecolor{mathgreen}{RGB}{0,204,0}
\definecolor{mathblue}{RGB}{0,0,204}

\newcommand{\red}[1]{{\color{red}\bf#1}}

\newcommand{\darkgreen}[1]{{\color{darkgreen}\bf#1}}

\renewcommand{\emph}[1]{{\textbf{#1}}}
\newenvironment{myproof}[1][\proofname]{\proof[#1]\mbox{}\\*}{\endproof}

\begin{document}

\title{Nonparametric multiple change point estimation \\ in highly dependent time series}
\author{Azadeh Khaleghi \\Mines ParisTech \and Daniil Ryabko \\ INRIA, Lille}
\date{}
\maketitle

\begin{abstract}
Given a  heterogeneous time-series sample, 
the objective is to find
points in time (called change points) 
where the probability distribution generating the data has changed. 
The data are assumed to have been generated by arbitrary unknown stationary ergodic distributions.
No modelling, independence or mixing assumptions are made.
A novel, computationally efficient, nonparametric method is proposed, and  
is shown to be asymptotically
consistent in this general framework.
The theoretical results are complemented with experimental evaluations.
\end{abstract}
\section{Introduction}\label{sec:inrto}
Change point estimation is a classical problem 
in mathematical statistics \cite{brodsky:93,basseville:93} 
which, with its broad range of applications in learning problems, 
has started to gain attention  in  the machine learning community. 
The  problem can be introduced  as follows. 
A given sequence 
$$
\x:=
X_{1},\dots,X_{\lfloor n\theta_1\rfloor}, 
X_{\lfloor n\theta_1\rfloor+1},\dots,X_{\lfloor n\theta_2\rfloor}, \dots,
X_{\lfloor n\theta_{\kappa}\rfloor+1},\dots,X_n
$$
is formed as the concatenation of  $\kappa+1$ of non-overlapping segments, 
where $\kappa\in\N$ and  $0<\theta_1 <\dots<\theta_\kappa<1$.  
Each segment is generated by some unknown time-series (or process) distribution. 
The  distributions that generate every pair of consecutive segments 
are different.  The index $\lfloor n\theta_k \rfloor $ where   one 
segment ends and another starts is called a {\em change point}.
The parameters $\theta_k,~k=1..\kappa$ specifying the 
change points $\lfloor n\theta_k \rfloor$ are  unknown and have to be estimated. 

In a typical formulation of  the problem, the samples 
within each segment $X_{\lfloor n\theta_1\rfloor+1}..X_{\lfloor n\theta_2\rfloor}$ 
are assumed to be i.i.d.\ and the change 
is in the mean 
(see, e.g., \cite{csorgo1997limit} for a  review). 
In the literature on nonparametric change point methods for dependent data
the form of the change and/or the nature of 
dependence are usually restricted; for example, a setting 
of time series that satisfy strong mixing conditions is often considered~\cite{brodsky:93}. 
Moreover, the finite-dimensional marginals are almost exclusively assumed different 
\cite{Carlstein:93,Giraitis:95}. 
Such assumptions often do not hold in  real-world applications.   

>From a machine-learning perspective,  
change point estimation appears to be a difficult unsupervised learning problem: 
an algorithm is required to locate the changes in a given sequence without any examples of correct solutions. 

In this paper, we consider  
highly dependent time series, making as few assumptions as possible
on how the data are generated. 
The only assumption that we make 
is that each segment is generated by an 
unknown stationary ergodic process distribution.
The joint distribution over the samples can be otherwise arbitrary.  
We make no such assumptions as independence, finite memory or mixing.
The marginal distributions of any given size   
before and after the change may be the same:
the change refers to that in the time-series distribution.

The main result of this paper is an  asymptotically consistent algorithm for 
estimating all $\kappa$ parameters $\theta_k, k=1..\kappa$ 
simultaneously.  We assume that $\kappa$ is given, but the process distributions as well as the nature
of the change are unknown. 
An estimate $\hat{\theta}_k$ of a change point parameter $\theta_k$ is  
{\em asymptotically consistent} if it becomes arbitrarily 
close to $\theta_k$ in the limit, as the length $n$ of the sequence approaches infinity. 
However,  the problem is {\em offline} and $\x$ does not grow with time. 
Thus, the asymptotic regime only means that  
the error is arbitrarily small if the sequence is sufficiently long. 
Real-world scenarios that correspond to this formulation include, for example, genomic data,  sequences of stock-market values, high-resolution audio/video data, and all such long sequential observations with distributional changes, where the distributions are completely unknown, changes are arbitrary, but the segments are long.  

While the assumption that each segment is generated by a stationary ergodic process is already very general, it can 
be relaxed even further.  In particular, one relatively simple but meaningful 
 generalisation that we consider is that each process is 
asymptotically mean stationary ergodic. This generalisation allows us  to address the problem of 
gradual (as opposed to abrupt) change in the distribution.

In general, for stationary ergodic processes, rates of convergence are 
provably impossible to obtain; this concerns already the convergence if frequencies
to probabilities  \cite{Sheilds:96}. 
Thus, non-asymptotic results  cannot be obtained in this setting.   
On the other hand, this means that, unlike in more restricted settings, in our setting the algorithms
are forced not to rely on any rate of convergence guarantees.  
We see this as an advantage of the framework, as it means that the algorithms are applicable to a 
much wider range of situations.
Furthermore, in this setting 
it is provably impossible to estimate $\kappa$. This follows from 
the impossibility result of \cite{Ryabko:10discr}, which states that it is not possible to determine, 
even in the weakest asymptotic sense, whether two sequences have been generated by the same or by different stationary ergodic distributions.
Thus, in this paper we assume that  
$\kappa$ is known.

The case of $\kappa=1$ was addressed in \cite{Ryabko:103s}, 
where a simple consistent algorithm for estimating one change point was provided. 
The general case of $\kappa >1$ turns out to be much more complex. 
With the sequence containing multiple change points, the algorithm 
is required to simultaneously analyse multiple segments of the input sequence,
with no a-priori lower bound on their lengths. 
In this case the main challenge is to ensure that the algorithm is robust with
respect to segments of arbitrarily small length. 
The problem is considerably simplified if additionally a lower bound 
on the minimum separation 
of the change points is provided.
Indeed, the method of \cite{Ryabko:103s} for $\kappa=1$ also relies on the knowledge of such parameter, 
namely, a lower bound on the minimum distance of the change point from the two end-points. 
With this additional information, some inference can be made even in the case where $\kappa > 1$ 
is unknown. Specifically, an algorithm is proposed in \cite{khaleghi:12mce} which, 
without the knowledge of $\kappa$, gives an exhaustive list of candidate estimates
whose first $\kappa$ elements are asymptotically consistent.
In this work we do not assume that a lower bound on the minimum separation 
of the change points is known.

Our algorithm is based on empirical estimates 
of the so-called distributional distance \cite{Gray:88}, 
which have proven useful in various statistical learning problems involving stationary ergodic time series
\cite{Ryabko:103s,khaleghi:12mce,Ryabko:10clust,Khaleghi:12,Ryabko:121c}. 
The computational complexity of our algorithm is at most quadratic in each argument.
We evaluate the proposed method on synthetic data
generated by processes that, while being stationary
ergodic, do not belong to any of the  ``simpler" classes studied in the literature on such problems, 
and cannot be modelled as hidden Markov processes with a countable set of states. 
Moreover, in the considered examples the single-dimensional
marginals before and after each change point are the same.

\noindent{\bf Organisation.}
In Section~\ref{sec:pre} we introduce preliminary notations and  definitions. 
In Section~\ref{sec:protocol} we formalise the problem and describe the general framework considered.  
In Section~\ref{sec:results}  we 
present our method, state the main consistency result, and informally describe how the algorithm works; the proof of the main
result is deferred to~Section~\ref{sec:proofs}.  
In Section~\ref{sec:exp} we provide some experimental evaluations. 
In Section~\ref{sec:extensions} we discuss some theoretical extensions of the considered framework and finally
we conclude in Section~\ref{sec:conc}.
\section{Preliminaries}\label{sec:pre}
Let  $\X$ be a measurable space (the domain); in this work we let $\X=\mathbb R$,
but extensions to more general spaces are straightforward.
For a sequence $X_1,\dots,X_n$ we use the abbreviation $X_{1..n}$.
Consider the Borel $\sigma$-algebra $\B$ on $\X^\infty$ generated by the cylinders 
$\{B\times \X^\infty: B\in B^{m,l}, m,l\in\N\}$,  
where   the sets $B^{m,l}, m,l \in \N$ are obtained via the partitioning of $\X^m$ into  cubes  
of dimension $m$ and volume $2^{-ml}$ (starting at the origin). Let also
$B^m:=\cup_{l\in\N}B^{m,l}$. 
Process distributions are probability measures 
on the space $(\X^\infty,\B)$.
For  $\x = X_{1..n}\in \X^n$ and $B\in B^m$ let $\nu(\x,B)$ 
denote the  frequency with which $\x$ falls in~$B$, i.e. 
\begin{equation}\label{eq:nu}
\nu(\x,B):=  {{\frac{\mathbb I\{n \geq m \}}{n-m+1}}}  \sum_{i=1}^{n-m+1} \mathbb I \{ X_{i..i+m-1} \in B \}.
\end{equation}
A process $\rho$ is {\em stationary}
if for any $i,j\in 1..n$ and $B \in B^m,~m \in \N$, 
we have $\rho(X_{1..j} \in B)=\rho(X_{i..i+j-1} \in B).$
A stationary process $\rho$ is called {\em  stationary ergodic} if for all $B\in\mathcal B$ 
with probability~1 we have 
$\lim_{n\rightarrow\infty}\nu(X_{1..n},B) = \rho(B).$ 
By virtue of the ergodic theorem  this definition 
can be shown to be equivalent to the usual definition given in terms
of shift-invariant sets; see e.g., \cite{Gray:88,Csiszar:04}. 
\begin{defn}[Distributional Distance \cite{Gray:88}]
The  distributional distance between a pair of process distributions
$\rho_1,\rho_2$ is defined as follows
$$
d(\rho_1,\rho_2):=\sum_{m,l=1}^\infty w_m w_l \sum_{B\in B^{m,l}} \left|\rho_1(B)-\rho_2(B) \right|.
$$
We let  $w_j:=\frac{1}{j(j+1)}$, 
 but any summable sequence of positive weights may be used.
\end{defn}
In words, we partition the sets $\X^m$, $m\in\N$ into cubes of 
decreasing volume (indexed by~$l$) and take a weighted sum over 
the differences in probabilities of all the cubes in these partitions. 
Different generating sets (other than cubes) can be used to define the distributional
distance; here we chose cubes in order to facilitate the experimental setup. 

Smaller weights are given to larger $m$ and finer partitions.
We use  empirical estimates of this distance, where probabilities are replaced with frequencies:
\begin{defn}[Empirical estimates of $d(\cdot,\cdot)$]\label{emd}
For $\x_i \in \X^{n_i}~n_i \in \N,~i=1,2$, and a distribution $\rho$ the empirical estimate of $d$ 
are defined as 
\begin{equation}\label{eq:emd1}
 \hat d(\x,\rho):=\sum_{m=1}^{m_n} \sum_{l=1}^{l_n}w_m w_l \sum_{B\in B^{m,l}} \left|\nu(\x,B)- \rho(B) \right|,
\end{equation}
\begin{equation}\label{eq:emd2}
 \hat d(\x_1,\x_2):=\sum_{m=1}^{m_n}\sum_{l=1}^{l_n} w_m w_l  \sum_{B\in B^{m,l}}  \left|\nu(\x_1,B)-\nu(\x_2,B)\right|,
\end{equation}
where  $m_n$ and $l_n$ are any sequences of integers that go to infinity with $n$.\\
\end{defn}
\noindent \textbf{Remark~1:} Despite the infinite summations, $\hat{d}$ can be calculated efficiently   \cite{Ryabko:10clust}.
Its  computational complexity is upper-bounded by
$\O(n~\text{polylog}~n)$ for $m_n:=\log n$, 
the choice of which is justified in \cite{Khaleghi:12} (see also \cite{khaleghi:12mce}). 
\begin{prop}[$\hat d(\cdot,\cdot)$ is consistent \cite{Ryabko:103s}]\label{thm:constd}
Let  a pair of sequences  $\x_1 \in \X^{n_1}$ 
and $\x_2 \in \X^{n_2}$  be generated by a distribution
$\rho$ whose marginals $\rho_i,~i=1,2$  are
stationary and ergodic. Then
\begin{align}
&\lim_{n_i \rightarrow\infty}\hat d(\x_i,\rho_j)=d(\rho_i,\rho_j),\ i,j \in 1,2,\ \rho-\as, \label{eq:const_2}\\
&\lim_{n_1,n_2\rightarrow\infty}\hat d(\x_1,\x_2)=d(\rho_1,\rho_2),\ \rho-\as \label{eq:const_1}
\end{align}
\end{prop}
\section{Problem formulation}\label{sec:protocol}
We formalise the problem of multiple change point estimation as follows.
The sequence $\x\in \X^n,~n \in \N$
is formed as the concatenation of  
$\kappa+1$ of 
sequences  
$$
X_{1..\lfloor n\theta_1\rfloor},
X_{\lfloor n\theta_1\rfloor+1..\lfloor n\theta_2\rfloor}, \dots,
X_{\lfloor n\theta_{\kappa}\rfloor+1..n},
$$
where  $\theta_k \in (0,1),~k=1..\kappa$, and where the number of change points $\kappa$ is assumed known. Denote $\theta_0:=0,~\theta_{\kappa+1}:=1$.
Each of the sequences $\x_k:=X_{\lfloor n\theta_{k-1}\rfloor +1..\lfloor n\theta_k\rfloor},~k=1..\kappa+1,$
is generated by an {\em unknown stationary ergodic} process distribution. 
Formally, consider a matrix 
$\bf{X} \in (\X^{\kappa+1})^{\infty}$
of random variables
generated by some (unknown)  stochastic process distribution $\rho$ 
such that 
\textbf{1.}~the marginal distribution over every one of its rows is
an unknown stationary ergodic process distribution;
\textbf{2.}~the marginal distributions over the consecutive rows are different, so that 
every two consecutive rows are generated by different process distributions. 
The sequence $\x \in \X^n$ is formed as follows. 
First, the length $n \in \N$ is fixed, next for each $k = 1..\kappa+1$  
a segment  
$\x_k \in \X^{\lfloor n(\theta_k - \theta_{k-1})\rfloor }$ is obtained as the first 
$\lfloor n(\theta_k - \theta_{k-1})\rfloor $ elements of the $k^{\text{th}}$ row of 
$\bf{X} $. 

Note that the requirements are only on the marginal distributions over the rows; 
the distribution $\rho$ is otherwise completely arbitrary. 
The process distributions are  unknown
and may  be dependent. Moreover, the
means, variances, or, more generally, the finite-dimensional
marginal distributions of any fixed size before and after the
change points are not required to be different. We consider the most general
scenario where   the {\em process distributions are different}.

The {\em unknown} parameters $\theta_k,~k=1..\kappa$ specify the change points $\lfloor n\theta_k \rfloor$,
which separate consecutive segments $\x_k,\x_{k+1}$ 
generated by  different process distributions. 
Define  the minimum separation of the change point parameters  as
\begin{equation}\label{defn:thetamin}
\mingap:=\min_{k=1..\kappa+1} \theta_k-\theta_{k-1}.
\end{equation}
Since the consistency properties we are after
are asymptotic in~$n$, we require that $\mingap>0$. 
Note that this condition is standard in the change point literature, 
although it may be unnecessary when simpler formulations of the problem 
are considered, for example when the samples within each segment  are i.i.d.
However, conditions of this kind are  
 inevitable in the general setting that we consider,
where the segments and the samples within each segment  
are allowed to be arbitrarily dependent:
 if the length of one of the sequences is constant or 
sub-linear in $n$ then asymptotic consistency is not possible in this setting. 
Finally, note that we make no assumptions on 
the distance between the process distributions:
they can be arbitrarily close. 

Our goal is to devise an algorithm 
that provides estimates $\hat \theta_k$ for the parameters $\theta_k,~k=1..\kappa$. 
The algorithm must be {\em asymptotically consistent} so that 
\begin{align}\label{eq:defn:mce}
&\lim_{n \rightarrow \infty} \sup_{k=1..\kappa}|\hat{\theta}_k(n)-\theta_k|=0\ \as
\end{align}
\section{Main result}\label{sec:results}
In this section we propose  Algorithm~\ref{alg:kk}, 
which, as shown in Theorem~\ref{thm:kk}, is asymptotically consistent
under the general assumptions stated in Section~\ref{sec:protocol}.
The proof of the consistency result is deferred to Section~\ref{sec:proofs}. 
Here we give an intuitive description as to how the algorithm works and why the consistency result holds. 
\begin{thm}\label{thm:kk}
Algorithm~\ref{alg:kk} is asymptotically consistent, provided that each segment $\x_k$, $k=1..\kappa$, is generated by a stationary 
ergodic distribution, and that the correct number $\kappa$ of
change points is given:
$$
\lim_{n \rightarrow \infty} \sup_{k=1..\kappa}\left|\hat{\theta}_k(n)-\theta_k\right|=0\ \as
$$
\end{thm}
The following two operators, 
namely, the 
score function  $\Delta_{\x}$
and the single-change point-estimator $\Phi_{\x}$ 
are used in our method.
\begin{defn}
Let $\x=X_{1..n}$ be a sequence 
and consider a subsequence $X_{a..b}$ of $\x$ with $a < b \in 1..n$.  
\begin{enumerate}
\item[i.~] Define the score function or the intra-subsequence distance of $X_{a..b}$ as 
\begin{equation}\label{defn:Delta}
\Delta_{\x}(a,b):= \hat{d} \left(X_{a.. \lfloor \frac{a+b}{2}\rfloor},X_{\lceil \frac{a+b}{2}\rceil..b} \right)
\end{equation} 
\item[ii.~] Define the single-change point estimator of $X_{a..b}$ as
\begin{equation}\label{defn:Phi}
\Phi_{\x}(a,b,\alpha):=~\argmaxdisp{t \in a..b}~\hat{d}\left(X_{a-n\alpha..t},X_{t..b+n\alpha}\right), 
~\text{where $\alpha \in (0,1)$}
\end{equation} 
\end{enumerate}
\end{defn}
Let us start by giving an overview of what 
Algorithm~\ref{alg:kk} aims to do. The algorithm attempts to simultaneously estimate all $\kappa$ 
change points using the single-change point-estimator  $\Phi_\x$
given by \eqref{defn:Delta} applied to appropriate segments of the sequence. 
In order for $\Phi_\x$ to produce asymptotically consistent estimates in this setting, 
each change point must be isolated within 
a segment of $\x$ whose length is a linear function of $n$. 
Moreover, each segment containing a change point must be 
``sufficiently far'' from the rest of the change points, 
where ``sufficiently far'' means within a distance linear in $n$.
This may be achieved by 
dividing $\x$ into consecutive non-overlapping segments, each of 
length $n\alpha$ with $\alpha:=\gap/3$ for some $\gap \in (0,\mingap]$,
where $\mingap$ is given by \eqref{defn:thetamin}. 
Since, by definition, $\mingap$ specifies the minimum separation of the 
change point parameters, the resulting partition has the property that 
every three consecutive segments of the partition contain  {\em at most one}
change point. 
However, $\mingap$ is not known to the algorithm.
Moreover, even if $\gap \leq \mingap$,
not all segments in the partition contain a change point. 
The algorithm uses the score function $\Delta_\x$  given by (\ref{defn:Delta}) 
to identify the segments that contain change points. 
As for $\mingap$, instead  of trying to find it, the algorithm
produces many partitions of $\x$ (using different guesses of $\mingap$), and produces a set of candidate change point estimates using
each guess.
Finally, a weighted combination of the candidate estimates is produced.
The weights are designed to converge to zero on iterations where   
the algorithm's guess of a lower bound on $\mingap$ is incorrect. 
\begin{algorithm}[!t]
\caption{A multiple change point estimator}\label{alg:kk}
\begin{algorithmic}[1]
\State \textbf{input:} $\x = X_{1..n}$, Number $\kappa$ of Change points
\State \textbf{initialize:} $\eta \gets 0$
\For {$j=1..\log n$}
\State{$\gap_j \gets 2^{-j},~\alpha_j \gets \gap_j/3,~w_j \gets 2^{-j}$ } \Comment{Set the step size and iteration weight}
\For{$t=1..\kappa+1$}
\State {$b^{t,j}_i \gets n\alpha_j(i+\frac{1}{t+1}),~i=0..\lfloor \frac{1}{\alpha_j}-\frac{1}{t+1}\rfloor$} \label{alg:bounds}\Comment{Generate boundaries}
 \For{$l=0..2$}\label{alg:gamma0}
\State{$d_{i'} \gets \Delta_{\x}(b_{l+3(i'-1)}^{t,j},b_{l+3i'}^{t,j}),~i'=1..\frac{1}{3}(\lfloor \frac{1}{\alpha_j}-\frac{1}{t+1}\rfloor-l)$}\label{alg:di}
\State {$\gamma_l \gets d_{[\kappa]}$} \Comment{Store the $\kappa^{\text{th}}$ highest value}\label{alg:gamma_l}
\EndFor
\State {$\gamma(t,j)\gets \displaystyle \min_{l=0..2} \gamma_l$} \label{alg:gamma}
\Comment{Obtain the grid's performance score}\label{gamma:kk}
\State{$\{\mu_1,\dots,\mu_{\kappa}\} \gets \argmaxdispt{i \in 1..\lfloor \frac{1}{\alpha_j}-\frac{1}{t+1}\rfloor-1}{k=1..\kappa}\Delta_x(b_{i}^{t,j},b_{i+1}^{t,j})$}\Comment{\parbox[t]{.33\linewidth}{Find $\kappa$ segments of highest $\Delta_\x$; ($X_{b_{\mu_k}^{t,j}..b_{\mu_k+1}^{t,j}}$ is the segment with $k^\text{th}$ highest score).}}\label{alg:mu1}
\State{$(b_{[1]}^{t,j},\dots,b_{[\kappa]}^{t,j}) \gets \sort{(b_{\mu_1}^{t,j},\dots,b_{\mu_{\kappa}}^{t,j}})$}
\Comment{\parbox[t]{.4\linewidth}{Sort the selected boundaries in increasing order}}\label{alg:mu2}
\State{$\hat{\pi}_k^{t,j}:=\Phi_{\x}(b_{[k]}^{t,j},b_{[k]+1}^{t,j},\alpha_j),~k=1..\kappa$} 
\Comment{\parbox[t]{.4\linewidth}{Seek a change point 
in $\kappa$ segments of highest $\Delta_\x$ 
}}
\State{$\eta \gets \eta+w_j\gamma(t,j)$}\Comment{Update the sum of weights}
\EndFor
\EndFor 
\State{ $\hat{\theta}_k \gets \frac{1}{n\eta} \sum_{j=1}^{\log n}\sum_{t=1}^{\kappa+1}  w_j \gamma(t,j) \hat{\pi}_k^{t,j}
,~k=1..\kappa$} \Comment{Calculate the final estimates}\label{alg:estim}
\State \textbf{return:} $\hat{\theta}_1,\dots,\hat{\theta}_{\kappa}$
\end{algorithmic}
\end{algorithm}

More precisely, Algorithm~\ref{alg:kk} works as follows. 
Given $\x \in \X^n$, it
iterates over $j=1..\log n$, and at each iteration
it produces a guess $\gap_j$ as a lower-bound on $\mingap$. 
For every fixed $j$, a total of $\kappa+1$ grids are generated, 
each composed of evenly-spaced boundaries $b_i^{t,j},~i=0..\lfloor\frac{1}{\alpha_j}-\frac{1}{t+1}\rfloor$,
that are $n\alpha_j$ apart for $\alpha_j:=\gap_j/3,~\gap_j:=2^{-j}$. 
This is specified in Line~\ref{alg:bounds} 
of Algorithm~\ref{alg:kk}.
The grids have distinct starting positions $\frac{n\alpha_j }{t+1}$ for $t=1..\kappa+1$. 
As shown in the proof 
 of Theorem~\ref{thm:kk}, 
this ensures that for a fixed $j$ at least one of the grids for some $t \in 1..\kappa+1$ 
has the property that 
the change points do not lie at the boundaries. 
This idea is depicted in Figure~\ref{fig1}.
\begin{figure}
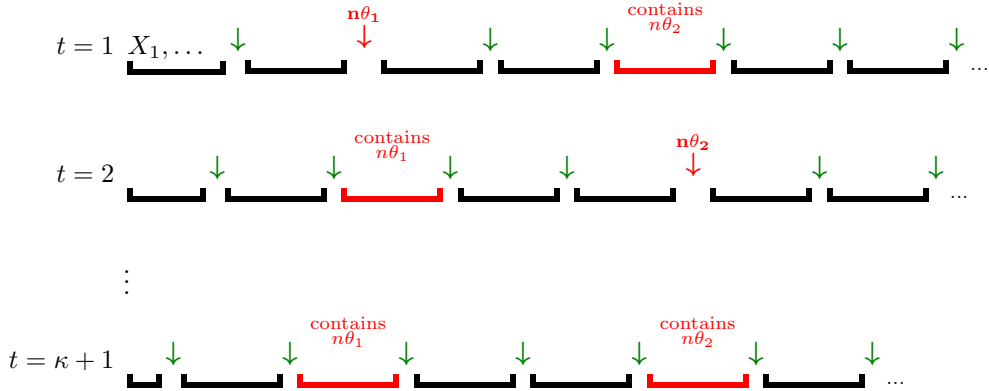

\begin{align*}
t=1~&\underbracket[2pt]{X_1,\dots~~}\overset{\darkgreen{\displaystyle \boldsymbol \downarrow}}{~} 
{\underbracket[2pt]{\textcolor{white}{\leftarrow X_1 \rightarrow}}}\overset{\red{\underset{\displaystyle \boldsymbol \downarrow}{n\theta_1}}}{}
{\underbracket[2pt]{\textcolor{white}{\leftarrow X_1 \rightarrow}}}\overset{\darkgreen{\displaystyle \boldsymbol \downarrow}}{}
{\underbracket[2pt]{\textcolor{white}{\leftarrow X_1 \rightarrow}}}\overset{\darkgreen{\displaystyle \boldsymbol \downarrow}}{}
\red{{\underbracket[2pt]{\overset{\substack{\text{contains}\\n\theta_2 }}{\textcolor{white}{\leftarrow X_1 \rightarrow}}}}}
\overset{\darkgreen{\displaystyle \boldsymbol \downarrow}}{}
{\underbracket[2pt]{\textcolor{white}{\leftarrow X_1 \rightarrow}}}
\overset{\darkgreen{\displaystyle \boldsymbol \downarrow}}{}
{\underbracket[2pt]{\textcolor{white}{\leftarrow X_1 \rightarrow}}}
\overset{\darkgreen{\displaystyle \boldsymbol \downarrow}}{~}
\underset{\dots}{\textcolor{white}{\dots}}
\\
\\
t=2~&\underbracket[2pt]
{\textcolor{white}{X_1}~~~~~}\overset{\darkgreen{\displaystyle \boldsymbol \downarrow}}{~} 
{\underbracket[2pt]{\textcolor{white}{\leftarrow X_1 \rightarrow}}}
\overset{\darkgreen{\displaystyle \boldsymbol \downarrow}}{~}
\red{{\underbracket[2pt]{\overset{\substack{\text{contains}\\n\theta_1 }}{\textcolor{white}{\leftarrow X_1 \rightarrow}}}}}
\overset{\darkgreen{\displaystyle \boldsymbol \downarrow}}{~}
{\underbracket[2pt]{\textcolor{white}{\leftarrow X_1 \rightarrow}}}
\overset{\darkgreen{\displaystyle \boldsymbol \downarrow}}{~}
{\underbracket[2pt]{\textcolor{white}{\leftarrow X_1 \rightarrow}}}
\overset{\red{\underset{\displaystyle \boldsymbol \downarrow}{n\theta_2}}}{~}
{\underbracket[2pt]{\textcolor{white}{\leftarrow X_1 \rightarrow}}}
\overset{\darkgreen{\displaystyle \boldsymbol \downarrow}}{~}
{\underbracket[2pt]{\textcolor{white}{\leftarrow X_1 \rightarrow}}}
\overset{\darkgreen{\displaystyle \boldsymbol \downarrow}}{~}
\underset{\dots}{\textcolor{white}{\dots}}
\\
\\
&\vdots\\
t=\kappa+1~&\underbracket[2pt]{\textcolor{white}{X_1}}
\overset{\darkgreen{\displaystyle \boldsymbol \downarrow}}{~} 
{\underbracket[2pt]{\textcolor{white}{\leftarrow X_1 \rightarrow}}}
\overset{\darkgreen{\displaystyle \boldsymbol \downarrow}}{~}
\red{{\underbracket[2pt]{\overset{\substack{\text{contains}\\n\theta_1 }}{\textcolor{white}{\leftarrow X_1 \rightarrow}}}}}
\overset{\darkgreen{\displaystyle \boldsymbol \downarrow}}{~}
{\underbracket[2pt]{\textcolor{white}{\leftarrow X_1 \rightarrow}}}
\overset{\darkgreen{\displaystyle \boldsymbol \downarrow}}{~}
{\underbracket[2pt]{\textcolor{white}{\leftarrow X_1 \rightarrow}}}
\overset{\darkgreen{\displaystyle \boldsymbol \downarrow}}{~}
\red{{\underbracket[2pt]{\overset{\substack{\text{contains}\\n\theta_2 }}{\textcolor{white}{\leftarrow X_1 \rightarrow}}}}}
\overset{\darkgreen{\displaystyle \boldsymbol \downarrow}}{~}
{\underbracket[2pt]{\textcolor{white}{\leftarrow X_1 \rightarrow}}}
\overset{\darkgreen{\displaystyle \boldsymbol \downarrow}}{~}
\underset{\dots}{\textcolor{white}{\dots}}
\end{align*}
\caption{For a fixed $j$, Algorithm~\ref{alg:kk} generates $\kappa+1$ 
grids composed of segments of length  
$n\alpha_j$ but with distinct starting points:
$n\alpha_j/(t+1),~t=1..\kappa+1$, where $\alpha_j$ is the algorithm's guess 
of $\mingap/3$. At the iteration shown in this figure, $\alpha_j\leq \mingap/3$
so that every three consecutive segments contain at most one change point. 
Since there are $\kappa$ change points, there exists at least one grid 
(in this example the one corresponding to $t=\kappa+1$) with the property 
that none of the change points are located at the boundaries. 
}
\label{fig1}
\end{figure}
Among the segments of the grid, 
$\kappa$ segments, $X_{b_{[k]}^{t,j}..b_{[k]+1}^{t,j}},~k=1..\kappa$, of highest score $\Delta_\x$ are selected; 
this is outlined in Lines~\ref{alg:mu1}~and~\ref{alg:mu2} of the algorithm.   
The single-change point estimator
$\Phi_\x$ is used to seek a candidate change 
point parameter in each of the selected segments. 
The weighted combination is given as the final estimate 
for every change point parameter $\theta_k,~k=1..\kappa$.
Two sets of weights are used,
namely, an iteration weight $w_j:=2^{-j}$ and a score 
$\gamma(t,j)$.
The former gives lower precedence to finer grids. 
To calculate the latter, at each iteration on $j$ and $t$,
for every fixed $l \in 0..2$, a partition of the grid is considered, 
composed of non-overlapping consecutive segments $X_{b_{l+3(i'-1)}^{t,j}..b_{l+3i'}^{t,j}},~i' = 1..
\frac{1}{3}(\lfloor\frac{1}{\alpha_j}-\frac{1}{t+1}\rfloor-l)$ of length $n\gap_j$. For each partition, a parameter $\gamma_l$ is calculated as 
the $\kappa^{\text{th}}$ highest intra-distance value $\Delta_x$ 
of its segments;
the performance weight $\gamma(t,j)$ is obtained as $\min_{l=0..2}\gamma_l$; 
this procedure is outlined in Lines~\ref{alg:gamma0}-\ref{gamma:kk} of the algorithm. 
(As shown in the proof,
 $\gamma(t,j)$  converges to zero 
on iterations where either $\gap_j>\mingap$ 
or there exists some change point on the boundary of one of the segments.)
\\
\textbf{Computational Complexity.~}
The proposed method can be easily and efficiently implemented. 
For a fixed $j$, a total of $1/\alpha_j$ 
distance calculations are done on segments of length $3\alpha_j$, 
and a total of $\kappa \alpha_j n$ distance calculations are done to 
estimate each change point; the procedure is repeated $\kappa+1$ times.  
By Remark~1, and summing over $j \in 1..\log n$ iterations,  
the overall complexity of these calculations is bounded by  $\O(\kappa^2n^2~\text{polylog}~n)$. 
The rest of the computations are of negligible order.
\section{Generalisation: AMS processes and gradual changes}\label{sec:extensions}
In this section we argue that our results can be strengthened to a more general case, where
the process distributions that generate the data 
are Asymptotically Mean Stationary (AMS) ergodic. 
We use this observation in turn to address
the problem of estimating gradual as opposed to abrupt changes in the distribution of the data. 

Recall that a process $\rho$ is {\em stationary}
if for any $i,j\in 1..n$ and $B \in B^m,~m \in \N$, 
we have $\rho(X_{1..j} \in B)=\rho(X_{i..i+j-1} \in B).$
A process $\rho$ is called {\em AMS} if for any $j \in 1..n$ and $B \in B^m,m~\N$ 
the series $\lim_{n\rightarrow \infty}\sum_{i=1}^{n}{1\over n}\rho(X_{i..i+j-1}\in B)$ converges.  In this case the limit,  
which we denote $\bar\rho(B)$, forms a measure   $\bar\rho(X_{1..j} \in B):=\bar\rho(B)$, $B \in B^m,~m \in \N$, 
 called {\em asymptotic mean} of $\rho$. Furthermore, for AMS processes for every  $B \in B^m,~m \in \N$, the frequency 
  $\nu(X_{1..n}, B)$ converges $\rho$-a.s.\ to a random variable 
  with mean $\bar\rho(B)$. Finally, as in the case of stationary processes, 
if the latter random variable is a.s.\  constant, then $\rho$ is called AMS ergodic. 
The reader is referred to \cite{Gray:88} for more information on AMS processes.

It is easy to check that 
our results readily hold for the case where 
the unknown process distributions that generate the data are AMS ergodic, 
and their asymptotic means before and after the change are different. 
Indeed, the  only  property that is used in the proofs is the convergence of all frequencies.
The class of all processes with AMS properties  is precisely the class of all processes for which this convergence holds.

This generalisation allows us to take into consideration gradual rather than abrupt changes in distribution.
So far we have considered a formulation in which the distribution is the same throughout a segment
and is different between the segments. This kind of change is referred to as {\em abrupt}.
Another formulation of the problem also considered in the literature (see e.g., \cite{brodsky:93}) is when 
the process distributions change gradually. 
More formally, we are given a sequence  $\x \in \X^n,~n \in \N$ such that 
$$
\x:=X_{1..\lfloor n\theta_1^{(1)}\rfloor},X_{\lfloor n\theta_1^{(1)}\rfloor+1..\lfloor n\theta_1^{(2)}\rfloor},
X_{\lfloor n\theta_1^{(2)}\rfloor+1..\lfloor n\theta_2^{(1)}\rfloor}, \dots,
X_{\lfloor n\theta_{\kappa}^{(2)}\rfloor+1..n}
$$
has $\kappa$ change points at $\lfloor n\theta_k^{(1)}\rfloor,~k \in \kappa$. 
The segments 
$X_{\lfloor n\theta_{k-1}^{(2)}\rfloor+1..\lfloor n\theta_{k}^{(2)}\rfloor},~k=1..\kappa$,
where $\theta_0^{(2)}:=0$ and $\theta_{\kappa+1}^{(1)}:=1$,
are generated by unknown, 
stationary ergodic process distributions. 
Moreover, their lengths are linear in $n$ so that
$\mingap:= \theta_k^{(2)}-\theta_{k-1}^{(2)} >0 $. 
The notion of gradual change is formalised by considering
between every pair of consecutive segments generated by different process distributions
 some arbitrary sequence of $o(n)$ length, i.e.
for all $k \in 1..\kappa$ we have 
$
\theta_k^{(2)}-\theta_{k}^{(1)} = o(n),
$ and $X_{\lfloor n\theta_1^{(1)}\rfloor+1..\lfloor n\theta_1^{(2)}\rfloor}$ is arbitrary (for example, deterministic).
Observe that under this formulation the process distributions generating the segments 
$X_{\lfloor n\theta_{k-1}^{(1)}\rfloor+1..\lfloor n\theta_{k}^{(1)}\rfloor},~k=1..\kappa$, where $\theta_0^{(1)}:=0$,
are AMS ergodic. Thus, by the above argument, the results of Theorem~\ref{thm:kk} 
carry over to this scenario as well, ensuring the asymptotic consistency of   
Algorithm~\ref{alg:kk}
in this formulation.  
Even more generally, under the AMS ergodic assumption, such $o(n)$ segments of arbitrary data  could be located anywhere within the segments;
however, the asymptotic consistency results would still remain unaffected.
\section{Experimental evaluations}\label{sec:exp}
In this section we evaluate our
method using synthetically generated data. 
In order to generate the data we use stationary ergodic process distributions that do not belong to any ``simpler" general class of time-series, and 
cannot be approximated by finite-state models. 
Moreover,  the single-dimensional
marginals of all distributions are the same throughout the generated sequence. 

We generate a segment $\y:=Y_1,\dots,Y_m \in \R^m,~m\in \N$ as follows. 
\textbf{1.}~Fix a parameter $\alpha \in (0,1)$ and two Gaussian distributions $\U_1$ and $\U_2$. 
\textbf{2.}~Let $r_0$ be drawn randomly from $[0,1]$.
\textbf{3.}~For each $i=1..m$ obtain 
$
r_i:=r_{i-1}+\alpha \mod 1
$
; draw $y^{(j)}_i$ from $\mathcal N_j,~j=1,2$.  
\textbf{4.}~Set 
$Y_i:=\mathbb{I}\{r_i\leq 0.5\}y_i^{(1)}+\mathbb{I}\{r_i> 0.5\}y_i^{(2)}.$
If $\alpha$ is irrational
this produces a real-valued stationary ergodic time-series. 
We simulate $\alpha$ by a long double with a long mantissa. 
Note that deterministically setting $y_i^{(1)}=0 $ and $y_i^{(2)}=1,~i\in 1..m$ 
results in a binary sequence $\x \in \{0,1\}^m$. 
Similar families  are  commonly used as examples in this framework, see, for example, \cite{Sheilds:96}.  

For the purpose of our experiments with the 
convergence of error-rate as a function of sequence-length $n$, 
we considered three values of  $\kappa$: $4,5~\text{and}~6$. 
In each case, we fixed $\kappa+1$ parameters 
$\alpha_1:=0.2..,~\alpha_2:=0.4..$, $\alpha_3:=0.6.., \dots$ 
(with long mantissae)\footnote{\begin{mytable}
{\tiny
\begin{tabular}{llll} 
$\alpha_1=0.22573625315372165312763512$ & $\alpha_2=0.465456356354654376453$&
$\alpha_3=0.678638276327863278362736283628736$&\\$\alpha_4=0.887438463874637846343$&
$\alpha_5=0.07283729372372987323232323$&$\alpha_6=0.4272638726382736328791217312893$&$\alpha_7=\alpha_1$\\
\end{tabular}
}\end{mytable}}
to correspond to different process distributions and 
used two Gaussian distributions $\mathcal N_1$ and $\mathcal N_2$ with means $0$ and $1$ respectively,
and standard deviation $1$.  
To produce $\x \in \R^n$ in each case we used the first
$\kappa$ change point parameters from the following sequence of $6$ values
$\theta_1=0.18$,  $\theta_2 =0.29$, $\theta_3=0.51$, $\theta_4 = 0.62$, $\theta_5=0.80$ and $\theta_6= 0.91$, 
and respectively set $\theta_0=0$ and $\theta_{\kappa+1}=1$. 
Notice that for $\kappa=6$ the minimum separation $\mingap$ between 
the change points is $0.09$ and for $\kappa=4,5$ it is $0.1$.   
Every segment of length 
$n_k:=\lfloor n(\theta_k-\theta_{k-1})\rfloor,~k=1..\kappa+1$ with $\theta_0:=0,~\theta_{\kappa+1}:=1$ 
was generated with $\alpha_{k},~k=0..\kappa+1$, and using $\U_1$ and $\U_2$. 
Figure~\ref{fig:synth} demonstrates the average estimation error of 
Algorithm~\ref{alg:kk}  
as a function of sequence length $n$. We calculate the estimation error as $\sum_{k=1}^{\kappa}|\hat{\theta}_{k}-\theta_k|.$
{\small \begin{figure}[!h]
\centering{
\includegraphics[scale=0.63]{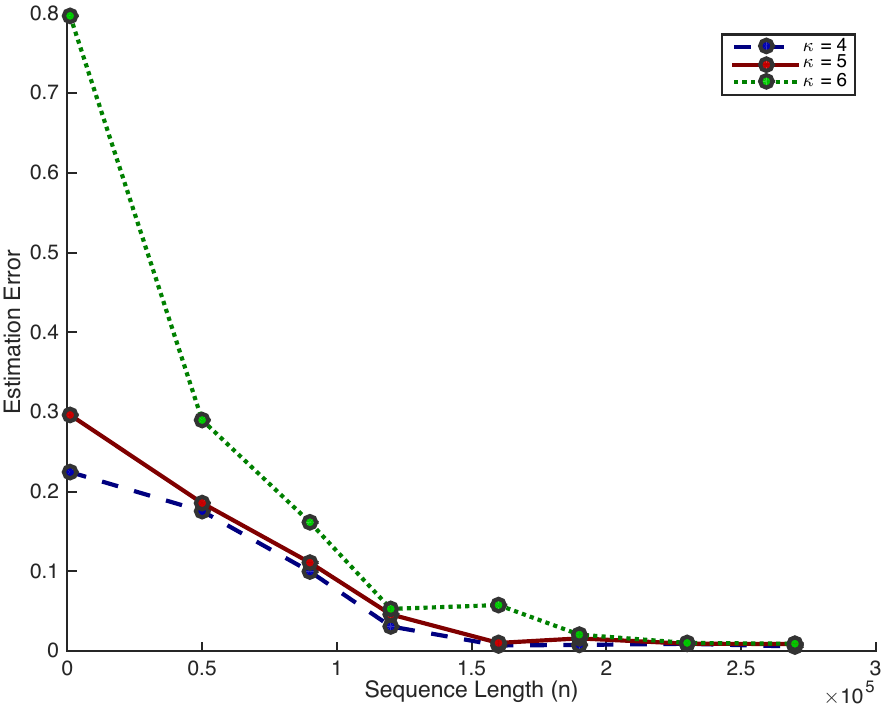}
}
\caption{Average (over 20 iterations) error of 
Alg\ref{alg:kk}$(\x,\kappa),~\x \in \R^n$, as a function of $n$ for 
$\kappa=4,5,6$.}
\label{fig:synth}
\end{figure}} 
As can be seen in the graph, while the algorithm converges in all three cases, the estimation error is on average 
slightly higher for larger $\kappa$. 
{\small \begin{figure}[!h]
\centering{
\includegraphics[scale=0.6]{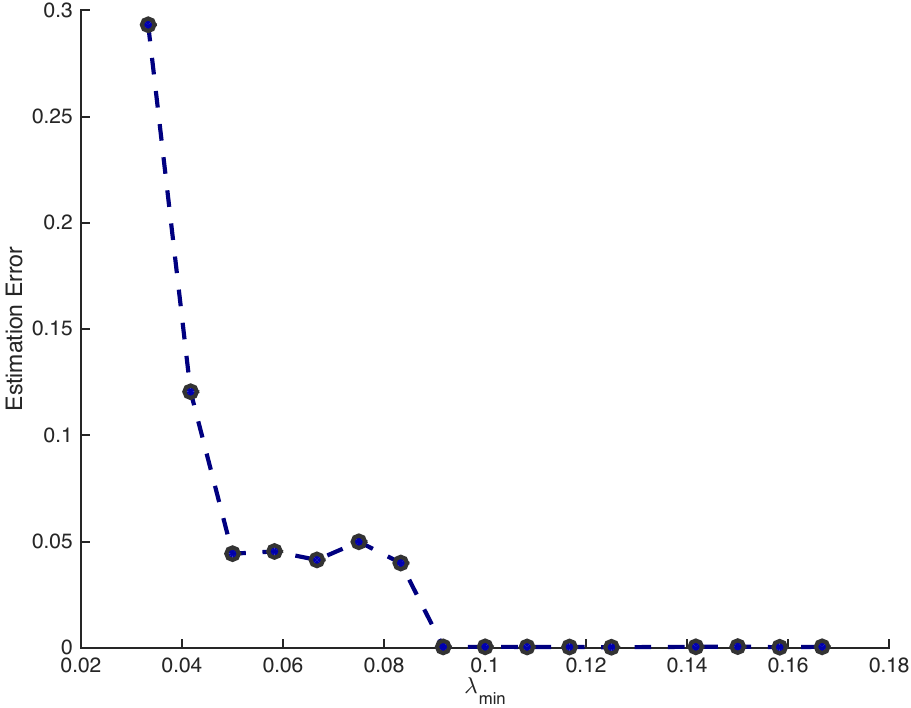}}
\caption{Average (over 10 iterations) error of 
Alg\ref{alg:kk}$(\x,\kappa),~\x \in \{0,1\}^n$ as a function of $\mingap$ for $n=3.0 \times 10^4$ and $\kappa=4$.}
\label{fig:lambda}
\end{figure}}

In order to examine the dependence of the algorithm on the minimum separation $\mingap$ 
between the change points, we fixed the sequence length $n$, 
and varied $\mingap$ to observe the average error change as a function $\mingap$. 
More specifically, we generated sequences $\x \in \{0,1\}^n,~n=3.0 \times 10^4$ with $\kappa = 4$ change-points
 as the concatenation of $5$  segments of lengths $5000,~n_0,12000-n_0,~7000,~6000$ respectively, where $n_0 = 1000, 1250, 1500, \dots, 5000$; note that $\mingap=n_0/n$ in this case.
To generate the segments we proceeded as in the previous experiment but with binary-valued processes (letting $y_i^{(1)}=0 $ and $y_i^{(2)}=1,~i\in 1..n$), and used
$\alpha_1:=0.12..,~\alpha_2:=0.14..$, $\alpha_3:=0.16.., \dots$
(with long mantissae)\footnote{\begin{mytable}
{\tiny
\begin{tabular}{llll} 
$\alpha_1=0.122573625315372165312763512$ & $\alpha_2=0.1465456356354654376453$&
$\alpha_3=0.1678638276327863278362736283628736$\\$\alpha_4=0.1887438463874637846343$&
$\alpha_5=0.107283729372372987323232323$&&\\
\end{tabular}
}
\end{mytable}} as parameters for the consecutive distributions. 
As can be seen in Figure~\ref{fig:lambda}, for a fixed $n$, the estimation error decreases as a function of $\mingap$.

\section{Concluding Remarks}\label{sec:conc}
We have presented an asymptotically consistent method to locate the changes
in highly dependent time-series data.
As explained in the introduction, in  the considered setting, rates of convergence (even of frequencies
to respective probabilities) are provably impossible to obtain, which 
is why the proposed algorithm comes only with asymptotic guarantees.
 At the same time, it may be interesting
to analyse how fast its error converges to zero  under stronger
assumptions, such as  i.i.d.\  or  mixing conditions. 
More generally, it would be interesting to discover whether 
asymptotic guarantees in the considered settings can be combined 
with optimality (up to constant factors) under stronger assumptions. 
This is left for future work.

\section{Proof of Theorem~\ref{thm:kk}}\label{sec:proofs}
In this section we prove the main consistency result. 
To facilitate the exposition, we  start with a sketch of the proof. 
\begin{myproof}[Proof Sketch]
To see why Algorithm~\ref{alg:kk} works, first observe that 
the empirical estimate $\hat{d}(\cdot,\cdot)$ of the distributional distance is 
consistent. Thus, the empirical distributional distance 
between a given pair of sequences converges to the 
distributional distance between their generating processes. 
From this we can show that 
the intra-subsequence distance $\Delta_\x$ 
corresponding to the segments in the grid that do not contain a change point  converges to zero. 
This is established in Lemma~\ref{prelem:nochpt:i}.(\ref{prelem:nochpt:ii}) below. 
On the other hand, since the generated grid becomes finer as a function of $j$, 
from some $j$ on, we have $\alpha_j <\mingap/3$ so that every three consecutive segments of the grid
contain \textit{at most} one change point. In this case, 
for  every segment that contains a change point, 
the single-change-point estimator $\Phi_\x$ produces an estimate 
that, for long enough segments, becomes arbitrarily close to the true change point. 
This is shown in Lemma~\ref{lem2}.\eqref{lem2:ii} below. 
Moreover, as follows from Lemma~\ref{lem2}.\eqref{lem2:i}, 
for large enough $n$ the performance 
scores associated with these segments are bounded below by some non-zero constant. 
Thus, the $\kappa$ segments of highest $\Delta_\x$ 
each contain a change point which can be  estimated consistently using $\Phi_\x$. 
However, the estimates produced at a given iteration for which $\alpha_j> \mingap/3$ 
may be arbitrarily bad.  
Moreover, recall that even for $\alpha_j\leq \mingap/3$, an appropriate grid to provide 
consistent estimates must have the property that 
no change point lies exactly on a grid boundary.  
However, it is not possible to directly identify 
such appropriate grids.
The following observation is key to their indirect identification. 

Consider the partitioning of $\x$ 
into $\kappa$ consecutive segments 
where   there exists at least one segment with more than one change point. 
Since there are exactly $\kappa$ change points, 
there must exist at least one segment in this partitioning that does not contain any change points. 
As follows from Lemma~\ref{prelem:nochpt:i}.(\ref{prelem:nochpt:ii}), 
the segment that contains no change points has an intra-subsequence 
distance $\Delta_{\x}$ that converges to~$0$. 
On the iterations for which $\alpha_j>\mingap/3$,  
at least one of the three partitions has the property that 
among every set of $\kappa$ segments in the partition, 
there is \textit{at least} one segment that  
contains no change points. In this case, $\Delta_{\x}$ 
corresponding to the segment without a change point converges to~$0$. 
The same argument holds for the case where    $\alpha_j \leq \mingap$, while at the same time a 
change point happens to be located
exactly at the boundary of a segment in the grid. 
Observe that for a fixed $j$, the algorithm forms a total of $\kappa+1$ different grids, with the same 
segment size, but distinct starting points $\frac{n\alpha_j }{t+1}~t=1..\kappa+1$. 
Since there are $\kappa$ change points, 
for all $j$ such that $\alpha_j \leq \mingap/3$ there exists at least one 
appropriate grid (for some $\tau \in 1..\kappa+1$), that simultaneously contains all the change points within its segments.   
In this case, $\gamma(\tau,j)$ converges to a non-zero constant. 
The final estimate $\hat{\theta}_k$ for each change point parameter $\theta_k$ is
obtained as a weighted sum of the candidate estimates produced at each iteration. 
Two sets of weights are used in this step, namely $\gamma(t,j)$ and $w_j$, 
whose roles can be described as follows. 
\begin{enumerate}
\item $\gamma(t,j)$ is used to penalise for the (arbitrary) results produced 
on iterations on $j \in 1..\log n$ and $t \in 1..\kappa+1$, where  
either $\alpha_j >\mingap/3$, or, while we have $\alpha_j \leq \mingap/3$,
there exists some $\theta_k$ for some $k \in 1..\kappa$ such that 
$\lfloor n\theta_k \rfloor \in \{b_i^{t,j}: i=0..\lfloor \frac{1}{\alpha_j}-\frac{1}{t+1}\rfloor\}$. 
As follows from the argument above, $\gamma(t,j)$ converges to zero only on these iterations, 
while it is bounded below by a non-zero constant on the rest. 
\item $w_j$ is used to give precedence to estimates sought 
in longer segments. Since the grids are finer for larger $j$, 
at some higher iterations the segments may not be long enough 
to produce correct estimates. 
\end{enumerate}
Therefore, if $n$ is large enough the final estimates $\hat{\theta}_k,~k=1..\kappa$ 
produced by Algorithm~\ref{alg:kk} converge to the true change point  parameters, $\theta_k,~k=1..\kappa$.
\end{myproof}
We now present the proof Theorem~\ref{thm:kk} which in turn depends upon some technical lemmas stated and proved below. 
\begin{lem}\label{prelem:nochpt:i}
Let $\x = X_{1..n}$ be generated by a stationary 
ergodic process $\rho$. 
For all $\alpha \in (0,1)$ 
the following statements hold with $\rho$-probability 1:  
\begin{enumerate}
\renewcommand{\theenumi}{\roman{enumi}}
\renewcommand{\labelenumi}{(\theenumi)}
\item \label{prelem:nochpt:i1}  
$\lim_{n \rightarrow \infty} \displaystyle \sup_{\substack{b_1,b_2 \in 1..n\\b_2-b_1 \geq \alpha n}}  
\sum_{\substack{B \in B^{m,l}\\m,l \in 1..T}} \dig{\rho}{b_1}{b_2}=0$
for every $T \in \N$.
\item \label{prelem:nochpt:i1d}
$\displaystyle \lim_{n \rightarrow \infty} \sup_{\substack{b_1,b_2 \in 1..n\\b_2-b_1 \geq \alpha n}}  
\hat{d}\left(X_{b_1..b_2},\rho \right)=0$.
\item\label{prelem:nochpt:ii}
$\displaystyle \lim_{n \rightarrow \infty} \sup_{\substack{b_2-b_1\geq \alpha n}} 
\Delta_{\x}\left(b_1,b_2\right) = 0  $.
\end{enumerate}
\end{lem} 
\begin{proof}
\eqref{prelem:nochpt:i1}. 
Assume the contrary:  
There exists and some $\lambda>0$, $T \in \N$ and 
sequences $b_1^{(i)} \in 1..n_i$ and $b_2^{(i)} \in 1..n_i,~n_i,i\in \N$ 
with $b_2^{(i)}-b_1^{(i)}\geq \alpha n_i$,
such that with probability $\Delta>0$ we have 
\begin{equation}\label{eq:contra}
 \limsup_{i \in \N}\sum_{\substack {B \in B^{m,l}\\m,l\in 1..T}}\left|\nu(X_{b_1^{(i)}..b_2^{(i)}},B)-\rho(B)\right| > \lambda.
\end{equation}
>From the definition of 
$\nu(\cdot,\cdot)$ given by~\eqref{eq:nu}, it is easy to see that 
for all $B \in B^{m,l},~m,l \in \N$ and $b_1<b_2 \in 1 ..n$ we have 
\begin{align}
\left|\nu(X_{b_1..b_2},B)-\rho(B) \right|
&\leq \left|\left(1-\frac{m-1}{b_2-b_1}\right)\nu(X_{b_1..b_2},B)-\rho(B)\right|+\frac{m-1}{b_2-b_1}\notag\\
&\leq \frac{4(m-1)}{b_2-b_1} + \sum_{i=1}^2 \frac{b_i}{b_2-b_1}\left|\nu(X_{1..b_i},B)-\rho(B)\right|.\label{eq:numan1}
\end{align}
Fix $\epsilon>0$.  
For each $m,l \in 1..T$ we can find a 
finite subset $S^{m,l} \subset B^{m,l}$ such that 
\begin{equation}\label{sml}
\rho(S^{m,l}) \geq 1-\frac{\epsilon}{T^2w_mw_l}.
\end{equation} 
Since $\rho$ is stationary ergodic, 
for every $B \in S^{m,l}$, there exists some $N(B)$ 
such that with probability 1 for all $t \geq N(B)$ we have  
\begin{equation}\label{eq:nu-rho} 
\sup_{b \geq t}\dig{\rho}{1}{b} \leq \frac{\epsilon\rho(B)}{T^2w_m w_l}. 
\end{equation} 
Define $\zeta_0:=\mindisp{m,l\in1..T}\frac{\epsilon}{T^2w_m w_l}$.
(Note that in the particular case where $w_i=1/i(i+1),~i=m,l$,  we simply get $\zeta_0=4\epsilon/T^2$, 
but we keep this parameter in its general form, independently of the specific choice of $w_m$ and $w_l$.)
Let $\zeta:=\min\{\alpha, \zeta_0\}$ and 
observe that $\zeta > 0$. 
For every $m,l \in 1..T$ and all $t \in \N$ we have,
\begin{align}\label{eq:sublin}
\sup_{\substack{b_1 \leq \zeta t\\b_2-b_1 \geq \alpha t}}\frac{b_1}{b_2-b_1} \leq \frac{\zeta }{\alpha}\leq\frac{\epsilon}{\alpha T^2w_mw_l}.
\end{align}
Define $N:=\maxdisp{m,l=1..T}N(B)/\zeta.$
On the other hand, by \eqref{eq:nu-rho} for all $n \geq N$ we have  
\begin{align}\label{eq:b11} 
\sup_{b_1 > \zeta n}\dig{\rho}{1}{b_1} \leq \frac{\epsilon\rho(B)}{T^2w_m w_l}.  
\end{align}
Increase $N$ if necessary to have 
\begin{equation}\label{eq:msml}
\sum_{m,l=1}^Tw_mw_l \frac{m}{\alpha N} \leq \epsilon.
\end{equation}
For all $n \geq N$ we obtain 
\begin{align}
&\sup_{\substack{b_1,b_2 \in 1..n\\b_2-b_1 \geq \alpha n}}
\sum_{m,l=1}^Tw_mw_l\sum_{B \in B^{m,l}} \dig{\rho}{b_1}{b_2} \notag \\
&\leq \sup_{\substack{b_1,b_2 \in 1..n\\b_2-b_1 \geq \alpha n}}\sum_{m,l=1}^Tw_mw_l\sum_{B \in S^{m,l}}\dig{\rho}{b_1}{b_2} 
+\epsilon \label{ineq0}\\
&\leq \sup_{\substack{b_1,b_2 \in 1..n \\b_2-b_1 \geq \alpha n}}\sum_{m,l=1}^Tw_mw_l\sum_{B \in S^{m,l}}\frac{b_2}{b_2-b_1}
\left|\nu(X_{1..b_2},B)-\rho(B) \right|\notag \\
&~+\sup_{\substack{b_1 > \zeta n\\b_2-b_1 \geq \alpha n}}\sum_{m,l=1}^Tw_mw_l\sum_{B \in S^{m,l}} \frac{b_1}{b_2-b_1}\left|\nu(X_{1..b_1},B)-\rho(B) \right|\notag\\
&~+\sup_{{\substack{b_1\leq \zeta n \\ b_2-b_1 \geq \alpha n}}}\sum_{m,l=1}^Tw_mw_l\sum_{B \in S^{m,l}} \frac{b_1}{b_2-b_1}\left|\nu(X_{1..b_1},B)-\rho(B) \right|
+5\epsilon \label{ineq11}\\
& \leq\epsilon(3/\alpha+5), \label{ineq5}
\end{align}
where \eqref{ineq0} follows from \eqref{sml}, 
\eqref{ineq11} follows from \eqref{eq:numan1} and \eqref{eq:msml}, 
and \eqref{ineq5} follows from 
\eqref{eq:nu-rho}, \eqref{eq:b11}, \eqref{eq:sublin},
summing over the probabilities, 
and noting that $\frac{b_2}{b_2-b_1} \leq \frac{1}{\alpha}$ for all 
$ b_2-b_1\geq\alpha n$. 
Observe that \eqref{ineq5}  holds for any $\epsilon>0$, 
and in particular for $\epsilon \in (0,\frac{\lambda}{3/\alpha+5})$. 
As a result, in the latter case for all $n \geq N$ we have
$$
\sup_{\substack{i \in \N \\ n_i \geq n}}\sum_{\substack {B \in B^{m,l}\\m,l\in 1..T}}\left|\nu(X_{b_1^{(i)}..b_2^{(i)}},B)-\rho(B)\right| < \lambda,
$$
contradicting~\eqref{eq:contra}. This contradiction  implies~\eqref{prelem:nochpt:i1}. 

\eqref{prelem:nochpt:i1d}
Fix $\epsilon >0$, $\alpha \in (0,1)$. 
We can find some $T \in \N$ such that 
\begin{equation}\label{prelem:nochpt:i1d:ml}
\sum_{m,l=T}^\infty w_m w_l \leq \epsilon.
\end{equation}  
By   \eqref{prelem:nochpt:i1}, there exists some $N$ 
such that for all $n \geq N$ we have 
\begin{align}\label{prelem:nochpt:i1d:nus}
\sup_{\substack{b_1,b_2 \in 1..n\\|b_2 -b_1| \geq \alpha n}}
\sum_{m,l=1}^T\sum_{B \in B^{m,l}}\dig{\rho}{b_1}{b_2}\leq \epsilon.
\end{align}
\\
>From \eqref{prelem:nochpt:i1d:ml} and \eqref{prelem:nochpt:i1d:nus},
for all $n \geq N$ we have
\begin{align*} 
\sup_{\substack{b_1,b_2 \in 1..n\\|b_2 -b_1| \geq \alpha n}}
 \hat{d}(X_{b_1..b_2},\rho ) 
&\leq \sup_{\substack{b_1,b_2 \in 1..n\\|b_2 -b_1| \geq \alpha n}}
\sum_{m,l=1}^{T}w_mw_l\sum_{B \in B^{m,l}}\dig{\rho}{b_1}{b_2}+\epsilon \\
&\leq 2\epsilon
\end{align*}
and the statement \eqref{prelem:nochpt:i1d} of the lemma follows. 

\eqref{prelem:nochpt:ii}
Fix $\epsilon >0$, $\alpha \in (0,1)$. 
Without loss of generality assume that $b_2 > b_1$. 
Observe that, for every $b_1+\alpha n\leq b_2 \leq n$,
we have $\frac{b_1+b_2}{2}-b_1 =  b_2-\frac{b_1+b_2}{2} \geq \alpha n/2$. 
Therefore, by  \eqref{prelem:nochpt:i1d},
there exists some $N$, 
such that for all $n \geq N_1$ we have 
\begin{align*}
\sup_{\substack{ b_2-b_1 \geq \alpha n}}
\hat{d}\left(X_{b_1..\frac{b_1+b_2}{2}},\rho \right)\leq \epsilon,\\
\sup_{\substack{b_2-b_1 \geq \alpha n}}
\hat{d}\left(X_{\frac{b_1+b_2}{2}}..b_2,\rho \right)\leq \epsilon.
\end{align*}
It remains to use the definition of $\Delta_\x$ given by (\ref{defn:Delta}) and the triangle inequality to observe that 
\begin{align*} 
\sup_{\substack{b_2-b_1 \geq \alpha n }} \Delta_{\x}\left(b_1,b_2 \right)
& = \sup_{\substack{b_2-b_1 \geq \alpha n }}  \hat{d}\left(X_{b_1..\frac{b_1+b_2}{2}},X_{\frac{b_1+b_2}{2}..b_2} \right) \\
&\leq \sup_{\substack{b_2-b_1 \geq \alpha n}} 
\hat{d}\left(X_{b_1..\frac{b_1+b_2}{2}},\rho \right)+\hat{d}\left(X_{\frac{b_1+b_2}{2}}..b_2,\rho \right)\notag 
\leq 2\epsilon
\end{align*}
for all $n \geq N$, and  \eqref{prelem:nochpt:ii} follows.

\end{proof}
\begin{lem}\label{prelem:chpt:dist}
Let $\x \in \X^n$ have a change point at 
$\pi =\theta n $ 
for some $\theta \in (0,1)$ so that the segments   
$X_{1..\pi}$, $X_{\pi..n}$ are generated by 
$\rho$, $\rho'$ respectively.  
If $\rho$, $\rho'$ are stationary ergodic, 
for every $\zeta \in (0,\min\{\theta,1-\theta\})$   
with probability 1
we have
\begin{enumerate}
\renewcommand{\theenumi}{\roman{enumi}}
\renewcommand{\labelenumi}{(\theenumi)}
\item \label{prelem:chpt:dist:i} 
$\displaystyle\lim_{n \rightarrow \infty}
\sup_{\substack{b \in 1..(\theta-\zeta)n \\t \in \pi..(1-\zeta)n}}
\hat{d}\left(X_{b..t},\frac{\pi-b}{t-b}\rho+\frac{t-\pi}{t-b}\rho' \right)=0,$
\item \label{prelem:chpt:dist:iii} 
$\displaystyle\lim_{n \rightarrow \infty}
\sup_{\substack{b \in \zeta n..\pi \\t \in (\theta+\zeta)n..n}}
\hat{d}\left(X_{b..t},\frac{\pi-b}{t-b}\rho+\frac{t-\pi}{t-b}\rho'\right)=0.$
\end{enumerate}
\end{lem}
\begin{proof}
\eqref{prelem:chpt:dist:i} Fix $\epsilon >0$, $\theta \in (0,1)$, $\zeta \in (0,\min\{\theta,1-\theta\})$. 
There exists some $T \in \N$ such that 
$\sum_{m,l=T}^{\infty} w_mw_l \leq \epsilon. $
By the definition of $\nu(\cdot,\cdot)$, for all 
$b \in 1..(\theta-\zeta)n,~t \in \pi..(1-\zeta)n$ and all 
$B \in B^{m,l}~m,l \in 1..T$ we have 
\begin{align}\label{prelem:chpt:dist:2nd}
\dig{\rho'}{\pi}{t}
&\leq \frac{n-\pi}{t-\pi-m+1}\left|\nu(X_{\pi..n},B)-\rho'(B)\right|\\
&\quad+\frac{n-t}{t-\pi-m+1}\left| \nu(X_{t..n},B)-\rho'(B)\right|
+\frac{3(m-1)}{t-\pi-m+1}.\notag
\end{align}
Furthermore, using  the fact that $\nu(\cdot,\cdot)\leq 1$, 
for all $b \in 1..(\theta-\zeta)n,~t \in \pi..(1-\zeta)n$ and
$B \in B^{m,l}~m,l \in 1..T$ we obtain 
\begin{align}
&\left|\nu(X_{b..t},B)-\frac{\pi-b}{t-b}\rho(B)-\frac{t-\pi}{t-b}\rho'(B)\right|\label{prelem:chpt:dist:chain:eq4}  \\
& \leq  \frac{\pi-b}{t-b}\dig{\rho}{b}{\pi}+\frac{t-\pi-m+1}{t-b}\dig{\rho'}{\pi}{t} 
+\frac{3(m-1)}{t-b}.\notag 
\end{align}
By Part \eqref{prelem:nochpt:i1} of Lemma~\ref{prelem:nochpt:i},
there exists some $N'$ 
such that for all $n \geq N'$ 
we have  
\begin{equation}\label{prelem:chpt:dist:1st}
\sup_{\substack{b \in 1..(\theta-\zeta)n}}\sum_{m,l=1}^Tw_mw_l\sum_{ B \in B^{m,l}}\dig{\rho}{b}{\pi} \leq \epsilon. 
\end{equation}
Similarly, $n-t\geq \zeta n$ for all $t \in \pi..(1-\zeta)n$. Therefore, 
by Part \eqref{prelem:nochpt:i1} of Lemma~\ref{prelem:nochpt:i}, there exists some $N''$
such that for all $n \geq N''$ we have 
\begin{equation}\label{prelem:chpt:dist:2nd0}
\sup_{\substack{t \in \pi..(1-\zeta)n}} \sum_{m,l=1}^Tw_mw_l\sum_{B \in B^{m,l}}
\dig{\rho'}{t}{n} \leq \epsilon. 
\end{equation}
Since $t-b \geq \zeta n$ for all $b \in 1..(\theta-\zeta)n,~t \in \pi..(1-\zeta)n$, 
we have  
$\frac{n}{t-b}\leq \frac{1}{\zeta}$.
For all $n \geq \frac{T}{\epsilon \zeta},~m \in 1..T$, 
$b \in 1..(\theta-\zeta)n$ and $t \in \pi..(1-\zeta)n$
we have 
$\frac{m-1}{t-b}\leq \frac{m}{\zeta n}\leq \epsilon.$ 
Let $N := \max\{N',N'',\frac{T}{\epsilon \zeta}\}$. 
By \eqref{prelem:chpt:dist:2nd}, \eqref{prelem:chpt:dist:chain:eq4}, \eqref{prelem:chpt:dist:1st}, \eqref{prelem:chpt:dist:2nd0},  
for all $n \geq N$ we have  
{\begin{align*}
\sup_{\substack{b \in 1..(\theta-\zeta)n \\t \in \pi..(1-\zeta)n}}
&\sum_{m,l=1}^Tw_{m,l}\sum_{ B \in B^{m,l}}\left|\nu(X_{b..t},B)-\frac{\pi-b}{t-b}\rho(B)-\frac{t-\pi}{t-b}\rho'(B) \right| 
\leq 3\epsilon(2+\frac{1}{\zeta}).
\end{align*}} 
By this, and the definition of $T$, 
for all $n \geq N$ 
we have
$\sup_{\substack{b \in 1..(\theta-\zeta)n \\t \in \pi..(1-\zeta)n}}
\hat{d}(X_{b..t},\frac{\pi-b}{t-b}\rho+\frac{t-\pi}{t-b}\rho')
\leq \epsilon(7+\frac{3}{\zeta})$
and Part \eqref{prelem:chpt:dist:i} 
follows. 
The proof of the second part 
is analogous.
\end{proof}
\begin{lem}\label{lem2}
Consider a sequence $\x \in \X^n,~n \in \N$ with $\kappa$ change points.
Let $\boldsymbol{b}:=b_1,\dots,b_{|\boldsymbol{b}|} \in \cup_{i=1}^n\{1..n\}^i$, 
be a sequence of indices with  
$\mindisp{i \in 1..|\boldsymbol b|-1}b_{i+1}-b_{i} \geq \alpha n$ for some $\alpha \in (0,1)$, such that 
for some $\zeta \in (0,1)$ we have
$\inf_{\substack{k=1..\kappa, b \in \boldsymbol{b}}}|\frac{1}{n}b-\theta_{k}| \geq \zeta .$
\begin{enumerate}
\renewcommand{\theenumi}{\roman{enumi}}
\renewcommand{\labelenumi}{(\theenumi)}
\item \label{lem2:i} 
With probability 1 we have
$\displaystyle \lim_{n \rightarrow \infty }\inf_{k\in1..\kappa}\Delta_{\x}(\lft{k}, \rt{k}) \geq \delta \zeta $
where $\lft{k}:= \displaystyle \maxdisp{\substack{b \leq n\theta_k}}\{ b \in \boldsymbol b\}$ and  
$\rt{k}:= \displaystyle \maxdisp{\substack{b > n\theta_k}} \{b \in \boldsymbol b\}$
denote the elements of $\boldsymbol b$ 
that appear immediately to the left and to the right of $\lfloor n\theta_k \rfloor$ respectively, 
and $\delta$ is the minimum distance between the distinct distributions that generate $\x$.
\item \label{lem2:ii} 
Assume that we additionally have 
$[\frac{1}{n}\lft{k}-\alpha,\frac{1}{n}\rt{k}+\alpha] \subseteq [\theta_{k-1},\theta_{k+1}]$. 
With probability 1 we obtain
$\displaystyle \lim_{n \rightarrow \infty}\sup_{k \in 1..\kappa}
|\frac{1}{n}\Phi_{\x}(\lft{k}, \rt{k},\alpha) -\theta_k| =0.$
\end{enumerate}
\end{lem}
\begin{proof}
(\ref{lem2:i}). Fix some $k \in 1..\kappa$. 
Define $c_k:=\frac{\llft{k}+\rrt{k}}{2}$. 
To prove Part \eqref{lem2:i}, 
we show that with probability $1$ for large enough $n$, we have  
\begin{equation}\label{objective1}
\hat{d}\left(X_{\llft{k}..c_k},X_{c_k..\rrt{k}} \right) \geq \delta\zeta.
\end{equation}
Fix $\epsilon >0$. 
Let $\pi_k:=\lfloor n\theta_k\rfloor, ~k=1..\kappa$. 
To prove \eqref{objective1}
for the case where   $\pi_k \leq c_k$ we proceed as follows. 
As follows from the assumption of the lemma
and the definition of $\lft{\cdot}$ and $\rt{\cdot}$, 
we have
$\rrt{k}-\llft{k} \geq n\alpha$,
so that 
$\rt{k}-c_k \geq \frac{\alpha}{2} n.$ 
Since by assumption of the lemma we have $\inf_{\substack{k=1..\kappa, b \in \boldsymbol{b}}}|\frac{1}{n}b-\theta_{k}| \geq \zeta$, it follows that $\pi_{k+1}-c_k \geq (\zeta+\frac{\alpha}{2})n.$
Moreover, from the same assumption
we have 
 $\frac{\pi_k-\llft{k}}{c_k-\llft{k}} \geq \frac{\pi_k-\llft{k}}{n} \geq \zeta.  $
Therefore, we obtain
\begin{align}\label{lem2:drhos}
& d\left(\rho_{k+1},\frac{\pi_k-\llft{k}}{c_k-\llft{k}}\rho_k+\frac{c_k-\pi_k}{c_k-\llft{k}}\rho_{k+1} \right)
= \frac{\pi_k-\llft{k}}{c_k-\llft{k}}d \left(\rho_{k+1},\rho_k \right) \geq \delta \zeta.
\end{align}
>From the definition of $\lft{k}$ and $\rt{k}$, and 
our assumption that $\pi_k \leq c_k$, 
the segment $X_{c_k..\rrt{k}}$ is fully generated by $\rho_{k+1}$. 
by Part~\eqref{prelem:nochpt:i1d} of Lemma~\ref{prelem:nochpt:i}, 
there exists some $N_1$ such that for all 
$n \geq N_1$ we have  
\begin{align}\label{lem2:i2ndhalf}
\hat{d}\left(X_{c_k..\rrt{k}},\rho_{k+1} \right) \leq \epsilon.
\end{align}
by Part \eqref{prelem:chpt:dist:i} of Lemma~\ref{prelem:chpt:dist} there exists some $N_2$ 
such that for all 
$n \geq N_2$ we have 
\begin{align}\label{lem2:i1sthalf}
\hat{d}\left(X_{\llft{k}..c_k},\frac{\pi_k-\llft{k}}{c_k-\llft{k}}\rho_k+\frac{c_k-\pi_k}{c_k-\llft{k}}\rho_{k+1}\right)\leq \epsilon.
\end{align}
By \eqref{lem2:i1sthalf} and  \eqref{lem2:drhos}
for all $n \geq \max_{i=1,2}N_i$ we obtain
\begin{align} 
&\Delta_{\x}\left(\llft{k},\rrt{k} \right) 
\geq  \hat{d}\left(X_{\llft{k}..c_k},\rho_{k+1} \right) - \hat{d}\left(X_{c_k..\rrt{k}},\rho_{k+1}\right)\notag \\  
&\geq  d\left(\rho_{k+1},\frac{\pi_k-\llft{k}}{c_k-\llft{k}}\rho_k+\frac{c_k-\pi_k}{c_k-\llft{k}}\rho_{k+1}\right)\label{lem2:final}\\ &\quad- \hat{d}\left(X_{\llft{k}..c_k}, \frac{\pi_k-\llft{k}}{c_k-\llft{k}}\rho_k+\frac{c_k-\pi_k}{c_k-\llft{k}}\rho_{k+1} \right) 
- \hat{d}\left(X_{c_k..\rrt{k}},\rho_{k+1} \right)
 \geq \delta \zeta -2\epsilon.\notag  
\end{align}
Since \eqref{lem2:final} holds for every $\epsilon>0$, 
this proves (\ref{objective1}) in the case where   $\pi_k \leq c_k$.
The proof for $\pi_k > c_k$ is analogous. 
Since \eqref{objective1} holds for all  $k \in 1..\kappa$, 
part \eqref{lem2:i} follows. 
\\
(\ref{lem2:ii}).~Fix some $k \in 1..\kappa$. 
Following the definition of $\Phi_\x$ given by (\ref{defn:Phi}) we have 
\begin{align*}
\Phi \left(\llft{k}-n\alpha,\rrt{k}+n\alpha,\alpha \right) 
: = \argmaxdisp{l' \in \llft{k}..\rrt{k}}~\hat{d}\left(X_{\llft{k}-n\alpha..l'},X_{l'..\rrt{k}+n\alpha}\right).
\end{align*}
We show that 
for any $\beta \in (0,1)$, with probability~$1$ for large enough $n$
we have  
\begin{align}
 \hat{d}\left(X_{\llft{k}-n\alpha..l'},X_{l'..\rrt{k}+n\alpha} \right)
 < \hat{d}\left(X_{\llft{k}-n\alpha..\pi_k},X_{\pi_k..\rrt{k}+n\alpha} \right).\label{objective}
\end{align} 
for all $l' \in \llft{k} .. (1-\beta) \pi_k \cup \pi_k(1+\beta)..\rrt{k}$. 
To prove (\ref{objective}) for $l' \in \llft{k} .. (1-\beta) \pi_k$ we proceed as follows. 
Fix some $\beta\in (0,1)$ and $\epsilon >0$. 
For all $l'\in\llft{k}..(1-\beta)\pi_k$ we have 
$\frac{\pi_k-l'}{\rrt{k}+n\alpha-l'}\geq \beta.$
Hence, by the definitions of $\hat{d}$ and $\delta$ we obtain  
\begin{align}\label{lem2ii:DR}
d\left(\rho_k,\rho_{k+1} \right) &-d\left(\rho_k,\frac{\pi_k-l'}{\rrt{k}+n\alpha-l'}\rho_k+\frac{\rrt{k}+n\alpha-\pi_k}{\rrt{k}+n\alpha-l'}\rho_{k+1} \right)
\geq \beta \delta.
\end{align}
by Part \eqref{prelem:nochpt:i1d} of Lemma~\ref{prelem:nochpt:i}, 
there exists some $N_1$ such that 
for all $n \geq N_1$ 
we have 
\begin{align}
\sup_{\substack{l' \in \llft{k}..\pi_k}}
&\hat{d}\left(X_{\llft{k}-n\alpha..l'},\rho_k \right) \leq \epsilon, \label{lem2ii:firstDiff1}\\
&\hat{d}\left(X_{\pi_k..\rrt{k}+n\alpha},\rho_{k+1} \right) \leq \epsilon.\label{lem2ii:firstDiff2}
\end{align}
For all $l' \in \lft{k}..\pi_k$ we have $l'-\pi_{k-1} \geq \alpha n$. 
Also, $\rt{k}+n\alpha \in \pi_k+n\alpha..\pi_{k+1}$. 
Therefore by Part \eqref{prelem:chpt:dist:iii} of 
Lemma~\ref{prelem:chpt:dist} there exists some $N_2$ such that 
\begin{align}\label{lem2ii:2ndiff}
\sup_{\substack{l' \in \llft{k}..\pi_k}}
\hat{d}\left(X_{l'..\rrt{k}+n\alpha},\frac{\pi_k-l'}{\rrt{k}+n\alpha-l'}\rho_k+\frac{\rrt{k}+n\alpha-\pi_k}{\rrt{k}+n\alpha-l'}\rho_{k+1}\right)\leq \epsilon.
\end{align} 
By (\ref{lem2ii:firstDiff1}), (\ref{lem2ii:firstDiff2}) and the triangle inequality, 
for all $n \geq  \max_{i=1,2} N_i$ we obtain 
\begin{align}\label{lem2ii:firsthalf}
\hat{d}\left(X_{\llft{k}-n\alpha..\pi_k},X_{\pi_k..\rrt{k}+n\alpha}\right)
\geq \hat{d}\left(\rho_k, \rho_{k+1}\right)-2\epsilon.
\end{align}
By \eqref{lem2ii:firstDiff1}, \eqref{lem2ii:2ndiff}, and using 
the triangle inequality, for all $n \geq  \max_{i=1,2} N_i$ 
we obtain  
\begin{align}
\sup_{\substack{l'\in\llft{k}..(1-\beta)\pi_k}} 
&\hat{d}\left(X_{\llft{k}-n\alpha..l'},X_{l'..\rrt{k}+n\alpha}\right) \notag \\
\leq \sup_{\substack{l'\in\llft{k}..(1-\beta)\pi_k}} 
&d\left(\rho_k,\frac{\pi_k-l'}{\rrt{k}+n\alpha-l'}\rho_k+\frac{\rrt{k}+n\alpha-\pi_k}{\rrt{k}+n\alpha-l'}\rho_{k+1}\right)+2\epsilon \label{lem2ii:2ndhalf}.
\end{align}
Finally, from (\ref{lem2ii:firsthalf}), (\ref{lem2ii:2ndhalf}) and (\ref{lem2ii:DR}) 
for all $n \geq  \max_{i=1,2} N_i$ we obtain 
{\begin{align}\label{lem2ii:final}
\inf_{\substack{l'\in\llft{k}..(1-\beta)\pi_k}}
&\hat{d}\left(X_{\llft{k}-n\alpha..\pi_k},X_{\pi_k..\rrt{k}+n\alpha}\right)\notag \\ &-\hat{d}\left(X_{\llft{k}-n\alpha..l'},X_{l'..\rrt{k}+n\alpha}\right)\geq \beta\delta-4\epsilon. 
\end{align}}

\noindent Since (\ref{lem2ii:final}) holds for every $\epsilon>0$, 
this proves (\ref{objective})  for $l' \in \llft{k} .. (1-\beta) \pi_k,~k\in1..\kappa$. 
The case where   $ l' \in (1+\beta)\pi_k..\rrt{k}$ is analogous; 
part \eqref{lem2:ii} follows. 
\end{proof}
\begin{proof}[Proof of Theorem~\ref{thm:kk}]
On each iteration on $j\in 1..\log n$ the algorithm 
produces a set of estimated change points. We show that on some iterations these
estimates are consistent, and that estimates produced on the rest of the iterations are negligible.
To this end, we will partition the set of iterations into three sets as described in Steps 1-3 below.  

Define 
$\zeta(t,j):=  \min_{\substack{k \in 1..\kappa, i \in 0..\lfloor \frac{1}{\alpha_j}-\frac{1}{t+1}\rfloor}}|\alpha_j(i+\frac{1}{t+1})  - \theta_k|$, 
$j=1..\log n,~t \in 1..\kappa+1$; 
for all $i=0..\lfloor \frac{1}{\alpha_j}-\frac{1}{t+1}\rfloor$ we have 
$|b_i^{t,j}-\pi_k| \geq n\zeta(t,j). $
\\ \noindent\textbf{Step 1.} 
Fix $\epsilon >0$.
There exist some $J_\epsilon$ such that 
 $\sum_{j=J_\epsilon}^\infty w_j \leq \epsilon. $
$J_\epsilon$ is used to cut off the 
the iterations over  $j\in[1..\log n]$  where $\gap_j$ is too small
for the estimates of the distributional distance between the segments  to be consistent (the grids are too fine).  
These iterations are penalised by small weights $w_j$, so that
the corresponding candidate estimates become negligible (their combined weight is less than $\epsilon$). \\
\noindent\textbf{Step 2.}
Let  $J(\mingap):=-\log (\mingap/3)$, 
where $\mingap$ is given by \eqref{defn:thetamin}.  
The iterations on $j$ for $j\in[J(\mingap),J_\epsilon]$ correspond to iterations  where $\gap_j \in (0,\mingap]$  and,
moreover, the segments are long enough for the estimates to be consistent as we show below. 
For all $j \geq J(\mingap)$ and $t \in 1..\kappa+1$, 
and every $\theta_k,~k \in 1..\kappa$ 
we have $[\frac{1}{n}\lft{k}-\alpha_j, \frac{1}{n}\rt{k}+\alpha_j] \subseteq [\theta_{k-1},\theta_{k+1}]$ 
where $\lft{\cdot}$ and $\rt{\cdot}$ are defined in Lemma~\ref{lem2}.
For every fixed $j \in J(\mingap)..J_\epsilon$ we identify a 
subset $\T(j)$ of the iterations on $t =1..\kappa+1$ at
which the change point parameters $\theta_k,~k=1..\kappa$ are estimated consistently 
and the performance scores $\gamma(t,j),~j \in J(\mingap)..J_\epsilon,~t\in \T(j)$ are bounded below by a nonzero constant. 
Moreover, we show that if the set $\T'(j):=\{1..\kappa+1\} \setminus \T(j)$ is nonempty, 
the performance scores $\gamma(t,j)$ for
all $j \in J(\mingap)..J_\epsilon$ and $t \in \T'(j)$ 
are arbitrarily small. \\
\textbf{i.} To define $\T(j)$ we proceed as follows. 
For every fixed $j \in J(\mingap)..J_\epsilon$, 
for  every $\theta_k,~k=1..\kappa$ we can uniquely define $q_k \in \N$ and $ p_k \in [0,\alpha_j)$ so that 
$\theta_k = q_k\alpha_j+p_k$.
Therefore, for any $p \in [0,\alpha_j)$ with $p \neq p_k,~k=1..\kappa$, 
we have 
$ \inf_{\substack{k=1..\kappa,~i \in \N\cup\{0\}}} |i\alpha_j+p-\theta_k|>0$.
Since we can only have $\kappa$ distinct residues $p_k,~k=1..\kappa$, 
any set of $\kappa+1$ different elements of $[0,\alpha_j)$
contains at least one element $p'$ such that $p'\neq p_k,~k=1..\kappa$. So, for every $j \in J({\mingap})..J_\epsilon$  
there exists at least one $t \in 1..\kappa+1$ 
such that $\zeta(t,j)>0$.
For every $j \in J(\mingap)..J_{\epsilon}$, 
define 
\begin{equation}\label{eq:T}
\T(j):=\left\{t \in 1..\kappa+1:\zeta(t,j)>0 \right\}.
\end{equation}
Let 
$ \bar{\zeta}(j) := \min_{t \in \mathcal T(j)}\zeta(t,j)$
and define
$\zeta_{\min}:= \inf_{j \in J(\mingap)..J_{\epsilon}}\bar{\zeta}(j)$.
Note that 
$\zeta_{\min}>0$.
By Part \eqref{lem2:i} of Lemma~\ref{lem2}, 
for all $j \in J(\mingap)..J_{\epsilon}$ there exists some $N_1(j)$ 
such that for all $n \geq N_1(j)$ we have 
\begin{equation}\label{thm:constJ1}
\inf_{t\in\T(j)}\gamma(t,j) \geq \delta\bar{\zeta}(j), 
\end{equation}
where $\delta$ is the minimum  
distance between the distinct distributions. 
As specified by Algorithm~\ref{alg:kk} we have 
$\eta := \displaystyle \sum_{j=1}^{\log n}\sum_{t=1}^{\kappa+1}w_j \gamma(t,j)$. 
By (\ref{thm:constJ1}) 
for all $n \geq N_1(J_{\mingap})$ we have
\begin{equation}\label{etabound}
\eta \geq w_{J(\mingap)}\delta\bar{\zeta}(J_{\mingap}),
\end{equation} 
which does not depend on $\epsilon$. 
By Part \eqref{lem2:ii} 
of Lemma~\ref{lem2},
there exists some $N_2(j)$ such that for all $n \geq N_2(j)$ we have 
\begin{equation}\label{thm1:eq:chptconst}
\sup_{\substack{k \in 1..\kappa, t \in 1..\T(j)}}\frac{1}{n}\left|\hat{\pi}_k^{t,j}-\pi_k \right|\leq \epsilon.
\end{equation}
\begin{figure}
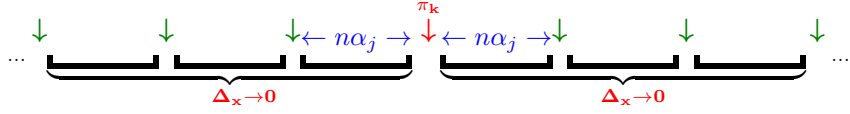

\begin{align*}
&~~~~\underset{\dots}{\textcolor{white}{\dots}}\overset{\darkgreen{\displaystyle \boldsymbol \downarrow}}{~}
\underset{\red{\Delta_x\rightarrow 0}}
{\underbrace{{{\underbracket[2pt]{\textcolor{white}{\leftarrow n\alpha_j \rightarrow}}}}
\overset{\darkgreen{\displaystyle \boldsymbol \downarrow}}{~}
{\underbracket[2pt]{\textcolor{white}{\leftarrow n\alpha_j \rightarrow}}}
\overset{\darkgreen{\displaystyle \boldsymbol \downarrow}}{~}
{\underbracket[2pt]{\textcolor{blue}{\leftarrow n\alpha_j \rightarrow}}}}
\overset{\red{\pi_k}}{\overset{\red{\displaystyle \boldsymbol \downarrow}}{~}}}
\underset{\red{\Delta_x\rightarrow 0}}
{\underbrace{{{\underbracket[2pt]{\textcolor{blue}{\leftarrow n\alpha_j \rightarrow}}}}
\overset{\darkgreen{\displaystyle \boldsymbol \downarrow}}{~}
{\underbracket[2pt]{\textcolor{white}{\leftarrow n\alpha_j \rightarrow}}}
\overset{\darkgreen{\displaystyle \boldsymbol \downarrow}}{~}
{\underbracket[2pt]{\textcolor{white}{\leftarrow n\alpha_j \rightarrow}}}}
\overset{\darkgreen{\displaystyle \boldsymbol \downarrow}}{~}}
\underset{\dots}{\textcolor{white}{\dots}}
\end{align*}
\caption{The case considered in Step~2.ii of the proof: $\lambda_j < \mingap$ and since the algorithm sets the spacing $\alpha_j$
between consecutive boundaries to $\lambda_j/3$, every three consecutive segments contain {\em at most one change point}. 
In the particular case depicted, one of the grid boundaries
lies exactly on some change point $\pi_k$. 
As follows from Step~2.ii, the grid score $\gamma$ assigned to 
such iterations  
converges to zero. 
}
\label{fig:hiitt}
\end{figure}

\noindent \textbf{ii.} Define $\T'(j):=\{1..\kappa+1\} \setminus \T(j)$ for $j \in J(\mingap)..J_{\epsilon}$, where $\T(j)$ is given by \eqref{eq:T}. 
It may be possible for the set $\T'(j)$ to be nonempty on some iterations on $j \in J(\mingap)..J_\epsilon$. 
Observe that by definition, for all $j \in J(\mingap)..J_\epsilon$ such that $\T'(j)\neq \emptyset$, we have 
$\max_{t \in \T'(j)} \zeta(t,j)=0$.  
This means that on each of these iterations, there exists some $\pi_k$ for some $k \in 1..\kappa$ 
such that $\pi_k =  b$ for some grid boundary
$$b \in \left\{b_i^{t,j}:=n\alpha_j(i+\frac{1}{t+1}),~i = 0.. \lfloor \frac{1}{\alpha_j}-\frac{1}{t+1}\rfloor ,~\alpha_j=\lambda_j/3,~t \in \T'(j)\right\}$$ where the boundaries are specified by Line~\ref{alg:bounds} of Algorithm~\ref{alg:kk}. 
Since $\gap_j \leq \mingap$ for all $j \in J(\mingap)..J_{\epsilon}$, and that $b=\pi_k$ 
we have  
$b..b+n\gap_j \subseteq \pi_k..\pi_{k+1}$ and $b-n\gap_j \subseteq \pi_{k-1}..\pi_k$. 
That is, the segments 
$X_{b..b+n\lambda_j}$ and $X_{b-n\lambda_j..b}$ 
are between two consecutive change points and are thus each  generated by a single process distribution. 
Following Lines \ref{alg:bounds} to \ref{alg:gamma_l} of Algorithm~\ref{alg:kk}, it is easy to see that 
in this case $\gamma(t,j)$ corresponds to $\max\{\Delta_\x(b,b+n\lambda_j), \Delta_\x(b-n\lambda_j, b)\}$. 
Since $b=\pi_k$, by 
Part \eqref{prelem:nochpt:ii} of Lemma~\ref{prelem:nochpt:i} 
there exists some $N_3(j)$ such that for all $n \geq N_3(j)$ we have  
$\max \{\Delta_{\x}(b-n\gap_j,b),\Delta_{\x}(b,b+n\gap_j)\} \leq \epsilon.$
Thus, for every $j \in J(\mingap)..J_{\epsilon}$ and all $n \geq N_3(j)$ 
we have 
\begin{equation}\label{thm:kk:gamma2}
\sup_{t\in \T'(j) \neq \emptyset} \gamma(t,j) \leq \epsilon.
\end{equation}
This scenario is depicted in Figure~\ref{fig:hiitt}. 

\noindent\textbf{Step 3.}
Consider,  $j =1..J(\mingap)-1$. 
It is desired for a grid to be such that 
every three consecutive segments contain at most one change point.
This property is not satisfied for $j = 1..J(\mingap)-1$ since, by definition, 
on these iterations we have $\alpha_j > \gap_j/3$. 
We show that for all these iterations, the performance score $\gamma (t,j),~1..\kappa+1$ 
becomes arbitrarily small; see Figure~\ref{fig:desire}. 
For all  $j=1.. J(\mingap)-1$ and $t=1..\kappa+1$, define the set of intervals
$\mathcal S^{t,j}:= \{(b_i^{t,j},b_{i+3}^{t,j}):i=0..\lfloor\frac{1}{\alpha_j}-\frac{1}{t+1}\rfloor-3\}$ and 
consider its partitioning into 
$$\mathcal S_l^{t,j} := \left\{ \left(b_{l+3i'}^{t,j},b_{l+3(i'+1)}^{t,j} \right):~i'=0..\frac{1}{3}\left( \left\lfloor\frac{1}{\alpha_j}-\frac{1}{t+1} \right\rfloor-l \right)\right\},~l=0..2.$$
Observe that, by construction, for every fixed $l=0..2$, every pair of indices $(b,b') \in \mathcal S_l^{t,j} $
specifies a segment $X_{b..b'}$ of length $3n\alpha_j$ 
and the elements of $\mathcal S_l^{t,j}$ index non-overlapping segments of $\x$.  
Since for all $j =1..J(\mingap)-1$ we have 
$\alpha_j > \gap_j/3,~j \in 1..J(\mingap)-1$ and $t\in1..\kappa+1$,
there exists some $(b,b') \in \mathcal S^{t,j}$ such that $X_{b..b'}$      
contains more than one change point. 
Since there are exactly $\kappa$ change points, 
in at least one of the partitions $\mathcal S^{t,j}_l$ for some $l \in 0..2$ 
we have that 
within any set of $\kappa$ segments 
there exists at least one segment that contains no change points. 
Note that, as specified by Lines~\ref{alg:bounds} - \ref{alg:gamma} of Algorithm~\ref{alg:kk}, we have
$$\gamma(t,j):=\min_{l=0..2} \left\{\Delta_{\x}(b,b'): (b,b') \in \mathcal S^{t,j}_l~\st~\left| \left\{(a,a') \in \mathcal S^{t,j}_l: \Delta_{\x}(a,a')>\Delta_{\x}(b,b') \right\} \right|=\kappa-1 \right\}.$$ Therefore, by Part~\eqref{prelem:nochpt:ii} of
Lemma~\ref{prelem:nochpt:i}, for every $j \in 1..J(\mingap)-1$ 
there exists some $N(j)$ such that for all $n \geq N(j)$ we have  
\begin{equation}\label{thm:kk:gamma}
\sup_{t \in 1..\kappa+1}\gamma(t,j) \leq \epsilon.
\end{equation} 
\begin{figure}
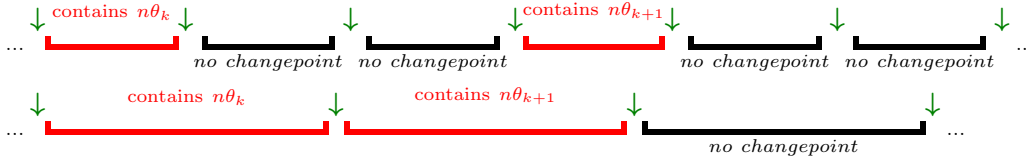

{
\begin{align*}
& \underset{\dots}{\textcolor{white}{\dots}}\overset{\darkgreen{\displaystyle \boldsymbol \downarrow}}{~}
      \red{\underset{}{\underbracket[2pt]{\overset{\text{contains}~n\theta_k }{\textcolor{white}{~\leftarrow \alpha n \rightarrow~}}}}}\overset{\darkgreen{\displaystyle \boldsymbol \downarrow}}{~}
            \underset{no~change point}{\underbracket[2pt]{\textcolor{white}{~\leftarrow \alpha n \rightarrow~}}}\overset{\darkgreen{\displaystyle \boldsymbol \downarrow}}{~}
      \underset{no~change point}{\underbracket[2pt]{\textcolor{white}{~\leftarrow \alpha n \rightarrow~}}}\overset{\darkgreen{\displaystyle \boldsymbol \downarrow}}{~}
            \red{\underset{}{\underbracket[2pt]{\overset{\text{contains}~n\theta_{k+1} }{\textcolor{white}{~\leftarrow \alpha n \rightarrow~}}}}}\overset{\darkgreen{\displaystyle \boldsymbol \downarrow}}{~}
                  \underset{no~change point}{\underbracket[2pt]{\textcolor{white}{~\leftarrow \alpha n \rightarrow~}}}\overset{\darkgreen{\displaystyle \boldsymbol \downarrow}}{~}
                  \underset{no~change point}{\underbracket[2pt]{\textcolor{white}{~\leftarrow \alpha n \rightarrow~}}}\overset{\darkgreen{\displaystyle \boldsymbol \downarrow}}{~}
      \underset{\dots}{\textcolor{white}{\dots}}\\
      & \underset{\dots}{\textcolor{white}{\dots}}\overset{\darkgreen{\displaystyle \boldsymbol \downarrow}}{~}
      \red{\underset{}{\underbracket[2pt]{\overset{\text{contains}~n\theta_k }{\textcolor{white}{~~\leftarrow \alpha n \rightarrow~~\leftarrow \alpha n \rightarrow~~}}}}}\overset{\darkgreen{\displaystyle \boldsymbol \downarrow}}{~}
            \red{\underset{}{\underbracket[2pt]{\overset{\text{contains}~n\theta_{k+1} }{\textcolor{white}{~~\leftarrow \alpha n \rightarrow~~\leftarrow \alpha n \rightarrow~~}}}}}\overset{\darkgreen{\displaystyle \boldsymbol \downarrow}}{~}
             \underset{no~change point}{\underbracket[2pt]{\textcolor{white}{~~\leftarrow \alpha n \rightarrow~~\leftarrow \alpha n \rightarrow~~}}}\overset{\darkgreen{\displaystyle \boldsymbol \downarrow}}{~}
      \underset{\dots}{\textcolor{white}{\dots}}
\end{align*}}
\caption{\textbf{Top.} Desired iteration where every three consecutive grid segments contain {\em at most one} change point. \textbf{Bottom.}  Undesired iteration where some groups of three consecutive grid segments may contain more than one change points. As follows from Step 2.i and Step~3, the algorithm indirectly distinguishes between the two scenarios. Specifically, in the former case the grid performance score $\gamma$ converges to a non-zero constant, while in the latter, it converges to zero.}
\label{fig:desire}
\end{figure}

\noindent {\bf Combining the steps.} 
Let $N:= \max \{\max_{j = 1..J(\mingap)-1} N(j), \max_{\substack{i=1..3 \\ j=J(\mingap)..J_{\epsilon}}}N_i(j)\}$
(note that the ranges of the $\max$ operators are finite, so $N$ is well defined).  
By \eqref{etabound}, the definition of $J_\epsilon$,  and that $\gamma(\cdot,\cdot) \leq 1$, 
for all $n \geq N$ we have 
\begin{align}\label{thm:lastThird}
\frac{1}{n\eta}\sum_{j=J_\epsilon}^{\log n}\sum_{t=1}^{\kappa+1}w_j\gamma(t,j) \left|\pi_k-\hat{\pi}_k^{t,j} \right|
\leq \frac{\epsilon(\kappa+1)}{w_{J(\mingap)}\delta\bar{\zeta}(J(\mingap))}. 
\end{align}
Note that  $\eta:=\sum_{j=1}^{\log n} \sum_{t=1}^{\kappa+1} w_j \gamma(t,j)$;
by \eqref{etabound}, \eqref{thm1:eq:chptconst} for all $n \geq N$ 
we have  
\begin{equation}\label{thm:middle1}
\frac{1}{n\eta}\sum_{j=J(\mingap)}^{J_{\epsilon}}\sum_{t\in \mathcal T(j)}w_j\gamma(t,j) \left|\pi_k-\hat{\pi}_k^{t,j} \right| \leq \epsilon. 
\end{equation}
By \eqref{etabound}, \eqref{thm:kk:gamma2} and \eqref{thm:kk:gamma} for all $n \geq N$
we obtain  
\begin{align}
&\frac{1}{n\eta}\sum_{j=J_\epsilon}^{\log n}\sum_{t \in \T'(j)}w_j\gamma(t,j)|\pi_k-\hat{\pi}_k^{t,j}|\leq \frac{\epsilon(\kappa+1)}{w_{J(\mingap)}\delta\bar{\zeta}(J(\mingap))},\label{thm:middle} \\
&\frac{1}{n\eta}\sum_{j=1}^{J(\mingap)-1}\sum_{t=1}^{\kappa+1}w_j\gamma(t,j)|\pi_k-\hat{\pi}_k^{t,j}|  
\leq \frac{\epsilon(\kappa+1)}{w_{J(\mingap)}\delta\bar{\zeta}(J(\mingap))}.\label{thm:firstThird}
\end{align}
Let $\hat{\theta}_k(n):=\frac{\hat{\pi}_k}{n},k=1..\kappa$.
By \eqref{thm:lastThird}, \eqref{thm:middle1}, \eqref{thm:middle} and \eqref{thm:firstThird} 
we have 
\begin{align*}
|\hat{\theta}_k(n)-\theta_k| 
&\leq \frac{1}{n\eta}\sum_{j=1}^{J(\mingap)-1}\sum_{t=1}^{\kappa+1}w_j\gamma(t,j)|\pi_k-\hat{\pi}_k^{t,j}|\\
&~~+\frac{1}{n\eta}\sum_{j=J(\mingap)}^{J_{\epsilon}}\sum_{t \in \mathcal T(j)}w_j\gamma(t,j)|\pi_k-\hat{\pi}_k^{t,j}|\\
&~~+\frac{1}{n\eta}\sum_{j=J(\mingap)}^{J_{\epsilon}}\sum_{t \in \T'(j)}w_j\gamma(t,j)|\pi_k-\hat{\pi}_k^{t,j}|\\
&~~+\frac{1}{n\eta}\sum_{j=J_\epsilon}^{\log n}\sum_{t=1}^{\kappa+1}w_j\gamma(t,j)|\pi_k-\hat{\pi}_k^{t,j}|  \\
&\leq \epsilon\left(1+\frac{3(\kappa+1)}{w_{J(\mingap)}\delta\bar{\zeta}(J(\mingap))}\right).
\end{align*}  
Since the choice of $\epsilon$ is arbitrary, the statement of the theorem follows.
\end{proof}
\subsection*{Acknowledgments}
\noindent This work was supported by the French Ministry of Higher Education and Research, by FP7/2007-2013 under grant agreements 270327 (CompLACS) and 216886 (PASCAL-2), 
 by the Nord-Pas-de-Calais Regional Council and FEDER
through CPER 2007-2013, and by an INRIA Ph.D. grant to Azadeh Khaleghi.

\end{document}